\documentclass[11pt]{article}

\usepackage{epsfig,epsf,fancybox}
\usepackage{amsmath}
\usepackage{mathrsfs}
\usepackage{amssymb}
\usepackage{graphicx}
\usepackage{color}
\usepackage{multirow}
\usepackage{paralist}
\usepackage{verbatim}
\usepackage{galois}
\usepackage{algorithm}
\usepackage[noend]{algpseudocode}
\usepackage{boxedminipage}
\usepackage{booktabs}
\usepackage{accents}
\usepackage{stmaryrd}
\usepackage[table]{xcolor}
\usepackage{hhline}

\usepackage{subfig}

\usepackage{natbib}

\usepackage{url}
\usepackage[colorlinks,linkcolor=magenta,citecolor=blue, pagebackref=true,backref=true]{hyperref}
\renewcommand*{\backrefalt}[4]{%
    \ifcase #1 \footnotesize{(Not cited.)}%
    \or        \footnotesize{(Cited on page~#2.)}%
    \else      \footnotesize{(Cited on pages~#2.)}%
    \fi}

\newtheorem{theorem}{Theorem}[section]

\newtheorem{lemma}[theorem]{Lemma}

\newtheorem{assumption}[theorem]{Assumption}

\numberwithin{equation}{section}

\newcommand{\EE}{\mathbb{E}}

\newcommand{\s}{\mathbf s}

\newcommand{\argmin}{\mathop{\rm argmin}}

\newcommand{\LCal}{\mathcal{L}}

\newcommand{\DCal}{\mathcal{D}}

\newcommand{\br}{\mathbb{R}}

\newcommand{\ba}{\begin{array}}
\newcommand{\ea}{\end{array}}

\newcommand{\FCal}{\mathcal{F}}

\begin{document}

\begin{center}

{\bf{\LARGE{The Row Normalization Puzzle in Muon}}}

\vspace*{.2in}
{\large{
\begin{tabular}{c}
Jiayu Zhang \and Tianyi Lin \\
\end{tabular}
}}

\vspace*{.2in}

\begin{tabular}{c}
Department of Industrial Engineering and Operations Research \\ 
Columbia University
\end{tabular}

\vspace*{.2in}

\today

\vspace*{.2in}

\begin{abstract}
This paper examines how row-wise renormalization affects Muon, focusing on the gap between NorMuon's worst-case guarantees and its practical performance~\citep{Li-2026-Normuon}. Despite its growing adoption and promising performance in large language model (LLM) pretraining, NorMuon's worst-case guarantees remain poorly understood. One fundamental question is: \textit{Does row normalization yield provable convergence gains, potentially through its interaction with approximate polar computation and exponential moving-average momentum?} Our results show that row normalization introduces a dimension-dependent factor in the worst-case iteration complexity under the operator-norm geometry, which persists even with exact polar computation and any fixed momentum parameters. Indeed, we establish an \textit{algorithm-dependent} lower bound and a matching upper bound in deterministic settings, and extend our upper bound analysis to stochastic settings. Both upper-bound analyses allow approximate polar computation. Experiments show that NorMuon is slower than Muon on synthetic problems inspired by our worst-case construction, yet outperforms Muon in LLM pretraining. These findings sharpen the puzzle of why row normalization helps in practice and complement the recent findings of~\citet{Dewulf-2026-Aurora}.
\end{abstract}
\end{center}

\section{Introduction}
Large-scale neural network training has driven major advances in vision and language~\citep{Krizhevsky-2012-Imagenet, He-2016-Deep, Vaswani-2017-Attention}, with large language model (LLM) pretraining placing increasing demands on optimization efficiency~\citep{Team-2025-Kimi}. Progress depends not only on data and computational resources, but also on optimizers that exploit the structure of neural network parameters. Among recent methods, Muon~\citep{Jordan-2024-Muon} has attracted attention for its training efficiency and scalability~\citep{Liu-2025-Muon}. Rather than applying coordinate-wise rescaling, Muon updates weight matrices using the polar factor of gradient momentum. Together with related methods such as Scion~\citep{Pethick-2025-Training}, it reflects a shift toward optimizer design based on input-output matrix geometry. More recently, NorMuon~\citep{Li-2026-Normuon} introduces row-wise second-moment normalization to balance update magnitudes across neurons and reports improved performance. Its practical relevance extends beyond optimizer benchmarks: NorMuon was used to train Datadog's Toto~2.0 forecasting models~\citep{Khwaja-2026-Toto}, while similar normalization has been incorporated into the nanochat LLM training pipeline~\citep{Karpathy-2025-Nanochat}. NVIDIA's Megatron Core also supports NorMuon-style updates as an optimizer option~\citep{NVIDIA-2026-Megatron}. These empirical gains motivate a closer examination of the convergence consequences of row normalization.

This design philosophy has a long history in optimization. Indeed, the simplex method exploits the polyhedral geometry of linear optimization~\citep{Dantzig-1963-Linear}, while gradient descent applies carefully designed step sizes in nonlinear optimization~\citep{Polyak-1963-Gradient}. Row normalization follows a related principle by respecting the neuron-wise organization of matrix parameters. Yet, geometric insights and practical success do not necessarily imply favorable worst-case guarantees. The simplex method is effective in practice, but needs exponentially many iterations in the worst case~\citep{Klee-1972-Good}. A related tension appears in gradient-based optimization. Step-size scheduling improves practical neural network training, from cosine annealing with warm restarts~\citep{Loshchilov-2017-SGDR} to warmup-stable-decay schedules for LLM pretraining~\citep{Wen-2025-Understanding}. In theory, step-size scheduling alone can accelerate gradient descent beyond its classical $O(T^{-1})$ rate~\citep{Altschuler-2025-Acceleration-I, Altschuler-2025-Acceleration-II}. However, recent work shows that optimally chosen schedules cannot attain the algorithm-independent lower bound of $\Omega(T^{-2})$~\citep{Ye-2026-Optimal}. We continue this line of thought for matrix optimizers, investigating whether the practical benefits of row normalization are reflected in its worst-case guarantees.

The benefits of row normalization have been systematically studied in recent work~\citep{Dewulf-2026-Aurora}, which identified a mechanism behind the imbalance in Muon's neuron-wise updates: neurons with low row leverage can receive persistently small updates and will become effectively inactive. NorMuon mitigates this issue by using exponential moving averages of squared row norms to rebalance update magnitudes across neurons. Their findings provide empirical evidence and alternative explanations for the practical value of balanced neuron updates, motivating a rigorous analysis of the worst-case convergence consequences of row normalization.

In this paper, we investigate row normalization through gradient-update alignment in operator-norm geometry. Our key construction uses a rank-one gradient with unequal row magnitudes: normalizing its polar factor balances the rows but weakens alignment with the gradient. We embed this mechanism in a smooth objective whose gradient remains constant along the iterates, isolating a directional distortion that persists under fixed momentum parameters and deterministic adaptive step sizes. A complementary descent analysis controls the same alignment together with momentum error and polar approximation.
Importantly, balanced neuron updates alone do not guarantee faster convergence. Our experimental results illustrate this distinction: NorMuon is slower than Muon on synthetic problems, yet outperforms it in LLM pretraining. Our findings sharpen the row normalization puzzle and further call for a deeper understanding of which key features of practical training make normalization beneficial despite its worst-case cost.

\paragraph{Contributions.} We focus on one fundamental question: does row normalization yield provable convergence gains, potentially through its interaction with approximate polar computation and exponential moving-average momentum? Our contributions can be summarized as follows:
\begin{enumerate}
\item We establish an algorithm-dependent lower bound of $\Omega(mL\epsilon^{-2})$ for NorMuon on smooth optimization over $\br^{m\times n}$, compared with Muon's $O(L\epsilon^{-2})$ rate in operator-norm geometry, where $L>0$ is the smoothness parameter. We prove upper bounds of $O(mL\epsilon^{-2})$ in the deterministic setting and $O(mL\epsilon^{-2}+m^2L\sigma^2\epsilon^{-4})$ in the stochastic setting, where $\sigma^2$ is a bound on the noise variance measured in the nuclear norm. Our analysis accommodates polar approximation. 
\item We conduct experiments on synthetic problems and LLM pretraining. The contrasting results separate the theoretical cost of row normalization from its practical benefits, complementing the findings of~\citet{Dewulf-2026-Aurora} and showing that improved neuron balance alone does not imply uniformly faster optimization.
\end{enumerate}
\paragraph{Related works.} Our work is most closely related to the literature on structured matrix-aware optimization and algorithm-dependent convergence analysis. Due to space limitations, we defer our comments on other relevant topics to Appendix~\ref{app:further-related-work}. Muon exploits matrix structure through polar-factor updates of gradient momentum, connecting spectral descent to efficient LLM pretraining~\citep{Jordan-2024-Muon, Bernstein-2024-Old, Liu-2025-Muon}. However, controlling singular values does not generally ensure balanced update magnitudes across neurons. To address this issue, NorMuon introduces row-wise second-moment normalization~\citep{Li-2026-Normuon}, while Muon+ applies stateless row/column normalization after orthogonalization~\citep{Zhang-2026-Muon}. Their ablation studies provide evidence of lower validation loss and highlight the importance of normalization direction. MuonEq instead normalizes momentum before orthogonalization, reporting lower validation loss and improved robustness~\citep{Chang-2026-Muoneq}. Dion3 incorporates row normalization into a variant with selective orthogonalization~\citep{Amsel-2026-Dion3}, while Aurora addresses low-leverage neurons in tall matrices, reporting reduced neuron death and lower validation loss~\citep{Dewulf-2026-Aurora}. These empirical findings motivate our analysis of the worst-case convergence cost of row normalization in NorMuon.

Algorithm-dependent convergence analysis studies specified update rules rather than oracle classes. Classical works have established sharp worst-case analyses for line search methods~\citep{Cartis-2010-Complexity, Cartis-2018-Worst} and many other first-order methods~\citep{Drori-2014-Performance, Taylor-2017-Smooth}. Recent works characterized the pros and cons of predetermined step-size schedules in gradient descent~\citep{Grimmer-2024-Provably, Altschuler-2025-Acceleration-I, Altschuler-2025-Acceleration-II, Ye-2026-Optimal}. For heavy-ball methods, the convergence analyses and cycling counterexamples~\citep{Ghadimi-2015-Global, Lessard-2016-Analysis} were followed by provable non-acceleration results for smooth strongly convex objectives~\citep{Goujaud-2025-Provable}, while acceleration becomes possible through over-relaxed recurrences~\citep{Wei-2025-Accelerated} or randomized scheduling on smooth convex objectives~\citep{He-2026-Randomized}. Our work establishes the dimension-dependent lower bound specifically for NorMuon with exact polar computation under any deterministic adaptive nonnegative step-size policy and any fixed momentum parameters. 

\section{Preliminaries and Technical Background}\label{sec:prelim}
We review the setup for LLM pretraining and the update rules of Muon and NorMuon, followed by the assumptions and basic properties used in our analysis.

\subsection{LLM pretraining and matrix optimization}
We consider an autoregressive language model with parameters $\theta$ and vocabulary $\mathcal{V}$. Given a token sequence $\s = (s_1,\ldots,s_{|\s|}) \in \mathcal{V}^{|\s|}$, the model assigns probability $p_\theta(\s)=\prod_{k=1}^{|\s|} p_\theta(s_k \mid s_{<k})$, where $s_{<k}$ denotes the first $k-1$ tokens. For a data distribution $\DCal$ over such sequences, the goal of pretraining is to minimize the following expected next-token prediction loss 
\begin{equation*}
\LCal(\theta) = -\EE_{\s \sim \DCal} \left[\tfrac{1}{|\s|} \sum_{k=1}^{|\s|} \log p_\theta(s_k\mid s_{<k})\right],
\end{equation*}
where the model parameters $\theta$ consist of weight matrices in attention and feed-forward layers~\citep{Vaswani-2017-Attention}. Muon and its variants apply their updates to each weight matrix separately. As such, we consider a single matrix variable in our theoretical analysis. Formally, we have
\begin{equation*}
\min_{X\in\mathbb{R}^{m\times n}} F(X),
\end{equation*}
where $F:\mathbb{R}^{m\times n}\to\mathbb{R}$ is differentiable but not necessarily convex.

Muon~\citep{Jordan-2024-Muon} updates a weight matrix using the polar factor of its gradient momentum. For a nonzero matrix $Z \in \br^{m\times n}$ with compact singular value decomposition (SVD) $Z = U\Sigma V^\top$, where $\Sigma \in \br^{r\times r}$ and $r=\operatorname{rank}(Z)$, we define $\operatorname{Polar}(Z)=UV^\top$, and set $\operatorname{Polar}(0)=0$. Let $\|\cdot\|_F$ denote the Frobenius norm. Starting from $M_0=0$, the exact Muon recursion is
\begin{equation*}
M_{t+1}=\beta M_t+(1-\beta)G(X_t,\xi_t),\quad O_{t+1}=\operatorname{Polar}(\tfrac{M_{t+1}}{\|M_{t+1}\|_F}),\quad X_{t+1}=X_t-\eta_t O_{t+1},
\end{equation*}
where $G(X_t,\xi_t) \in \br^{m\times n}$ is a stochastic gradient at $X_t$, $\beta \in [0,1)$ is the momentum parameter, and $\eta_t > 0$ is the step size. In practice, Newton-Schulz iterations~\citep{Jordan-2024-Muon} or PolarExpress~\citep{Amsel-2026-Polar} can approximate the polar factor using the normalized momentum $\frac{M_{t+1}}{\|M_{t+1}\|_F}$ as input.

NorMuon~\citep{Li-2026-Normuon} adds row-wise normalization to this update. Given the possibly approximate polar factor $O_{t+1}$, it maintains an exponential moving average of the mean squared entries of each row and rescales that row accordingly:
\begin{equation*}
v_{t+1,i} = \beta_2v_{t,i}+(1-\beta_2)\tfrac{\|[O_{t+1}]_{i,:}\|_2^2}{n},\quad [\overline{O}_{t+1}]_{i,:}=\tfrac{[O_{t+1}]_{i,:}}{\sqrt{v_{t+1,i}}+\alpha}, \quad \textnormal{for each } i=1,\ldots,m,
\end{equation*}
where $[O]_{i,:}$ denotes the $i^\textnormal{th}$ row of $O$, $\|\cdot\|_2$ denotes the Euclidean norm, $v_0=0$, $\beta_2 \in [0,1)$, and $\alpha \geq 0$ is a stabilization constant. Then, following the rescaling of~\citep{Liu-2025-Muon, Li-2026-Normuon}, we set
\begin{equation*}
D_{t+1}=\tfrac{0.2\sqrt{mn}}{\|\overline{O}_{t+1}\|_F}\overline{O}_{t+1}, \quad X_{t+1}=X_t-\eta_tD_{t+1}.
\end{equation*}
We see that row normalization changes the relative magnitudes of the rows, while the final rescaling fixes the Frobenius norm of every nonzero update. The full scheme is summarized in Algorithm~\ref{algorithm:normuon}, including approximate polar computation as specified in Assumption~\ref{assumption:polar}.
\begin{algorithm}[!t]
\caption{NorMuon}\label{algorithm:normuon}
\begin{algorithmic}[1]
\State \textbf{Input:} $X_0 \in \br^{m\times n}$, $\ell,\delta,\beta_1, \beta_2 \in [0,1)$, $\eta_t,\alpha\geq0$, $v_0=0\in\mathbb{R}^m$, and $M_0=0\in\mathbb{R}^{m\times n}$. 
\For{$t=0,\ldots,T-1$}
\State $G_t\gets G(X_t,\xi_t)$, where $\xi_t$ is independent of $\{X_0,\xi_0,\ldots,\xi_{t-1}\}$.
\State $M_{t+1} \gets \beta_1M_t+(1-\beta_1)G_t$.
\State $O_{t+1} \gets \operatorname{Polar}_{\ell,\delta}(\frac{M_{t+1}}{\|M_{t+1}\|_F})$.
\State $v_{t+1} \gets \beta_2v_t+\frac{1-\beta_2}{n}\operatorname{diag}(O_{t+1}O_{t+1}^\top)$.
\State $[\overline{O}_{t+1}]_{i,:}\gets\frac{[O_{t+1}]_{i,:}}{\sqrt{v_{t+1,i}}+\alpha}$ for all $i\in \{1,\ldots,m\}$.
\State $D_{t+1}\gets\tfrac{0.2\sqrt{mn}}{\|\overline{O}_{t+1}\|_F}\overline{O}_{t+1}$.
\State $X_{t+1}\gets X_t-\eta_tD_{t+1}$.
\EndFor
\Return $\widetilde{X} = X_{\widetilde{T}}$ where $\widetilde{T}$ is sampled independently and uniformly from $\{0,\ldots,T-1\}$. 
\end{algorithmic}
\end{algorithm}

\subsection{Assumptions and basic properties}
Throughout, we assume that $F^\star = \inf_{X \in \br^{m\times n}} F(X) > -\infty$ and let $\Delta\geq0$ be an upper bound on $F(X_0)-F^\star$. For $X \in \br^{m\times n}$, we define $\|X\|_\textnormal{op}=\max_{\|z\|_2 \leq 1} \|Xz\|_2$, and $\|X\|_\textnormal{nuc}$ as the sum of its singular values. These norms are dual under the Frobenius inner product $\langle X,Y\rangle=\operatorname{tr}(X^\top Y)$. Indeed, we have
\begin{equation*}
\max_{\|Y\|_\textnormal{op} \leq 1}\langle X,Y\rangle = \langle X,\operatorname{Polar}(X)\rangle = \|X\|_\textnormal{nuc}.
\end{equation*}
This implies that the exact Muon recursion minimizes a linear model over an operator-norm ball centered at $X_t$~\citep{Pethick-2025-Training}. In other words, we have $X_{t+1} \in \argmin_{\|X-X_t\|_\textnormal{op} \leq \eta_t} \langle X-X_t,M_{t+1}\rangle$. This motivates the operator-norm geometry used in our analysis. In particular, we impose the following assumptions on the objective and stochastic gradients.
\begin{assumption} \label{assumption:smooth}
There exists $L \geq 0$ such that $\|\nabla F(Y)-\nabla F(X)\|_\textnormal{nuc} \leq L\|Y-X\|_\textnormal{op}$ for all $X,Y \in \br^{m \times n}$.
\end{assumption}
\begin{assumption} \label{assumption:noise}
There exist $\sigma \geq 0$ and a stochastic gradient oracle $G: \br^{m\times n} \times \Xi \to \br^{m\times n}$ such that $\EE[G(X,\xi) \mid X] = \nabla F(X)$ and $\EE[\|G(X,\xi)-\nabla F(X)\|_\textnormal{nuc}^2 \mid X] \leq \sigma^2$ for all $X \in \br^{m\times n}$.
\end{assumption}
Assumption~\ref{assumption:smooth} implies the following descent inequality:
\begin{equation}
\label{eq:main-descent-ineq}
F(Y) \leq F(X)+\langle\nabla F(X),Y-X\rangle+\tfrac{L}{2}\|Y-X\|_\textnormal{op}^2, \quad \textnormal{for all } X,Y \in \br^{m\times n}. 
\end{equation}
Our analysis uses this inequality to relate the decrease in the objective to the alignment between the gradient and the update direction. In the deterministic setting, we have $G_t=\nabla F(X_t)$ and $\sigma=0$. To cover approximate polar computation, we impose a condition on the singular values of its output.
\begin{assumption} \label{assumption:polar}
There exist $\ell,\delta \in [0,1)$ and $\operatorname{Polar}_{\ell,\delta}:\br^{m\times n} \to \br^{m\times n}$ with $\operatorname{Polar}_{\ell,\delta}(0)=0$ such that the following statement holds. For every nonzero $Z \in \br^{m\times n}$ with $\|Z\|_F\leq 1$, let $Z=U\Sigma V^\top$ be a compact SVD, where $r=\operatorname{rank}(Z)$. Then $\operatorname{Polar}_{\ell,\delta}(Z)=U\widetilde{\Sigma}V^\top$ for some diagonal $\widetilde{\Sigma} \in \br^{r\times r}$ satisfying (i) $\widetilde{\Sigma}_{jj} \in [0, 1+\delta]$ and (ii) $\widetilde{\Sigma}_{jj} \in [1-\delta, 1+\delta]$ whenever $\Sigma_{jj} \geq \ell$, for all $j$. 
\end{assumption}
Assumption~\ref{assumption:polar} allows an approximate polar computation. If an input singular value is at least $\ell$, the output singular value must be within $\delta$ of $1$. For smaller input singular values, the corresponding output singular values lie between $0$ and $1+\delta$. As such, a smaller $\ell$ requires accurate approximation for more singular values, while a smaller $\delta$ requires greater accuracy. PolarExpress~\citep{Amsel-2026-Polar} uses polynomial iterations to adjust the singular values without changing the singular vectors and guarantees $\delta \leq (1-\ell^2)^{3^K}$ after $K$ iterations. 

We use $\min_{0\leq t<T}\|\nabla F(X_t)\|_\textnormal{nuc} \leq \epsilon$ and $\EE[\|\nabla F(X_{\widetilde{T}})\|_\textnormal{nuc}] \leq \epsilon$ to measure convergence in the deterministic and stochastic settings, and study the number of iterations that NorMuon requires to meet these criteria. In particular, we establish matching lower and upper bounds in the deterministic setting (Theorems~\ref{theorem:normuon-lower} and~\ref{theorem:normuon-upper-deterministic}), showing that row normalization incurs an extra factor of $m$ compared with Muon. We extend the upper bound to the stochastic setting (Theorem~\ref{theorem:normuon-upper-stochastic}), allowing approximate polar computation in both upper bound analyses.

\section{Main Results}\label{sec:results}
We establish matching upper and lower bounds for NorMuon in the deterministic setting and extend the upper bound to stochastic settings. We also explain the main ideas behind our analysis, with full proofs deferred to Appendix~\ref{app:proofs}. 

\subsection{Deterministic setting}
We establish an algorithm-dependent lower bound with exact polar computation. This isolates the effect of row normalization from errors in computing the polar factor.
\begin{theorem} \label{theorem:normuon-lower}
Let $m \geq 2$, $n \geq 1$, $\Delta,L>0$, and $\epsilon \in (0,\frac{1}{4}\sqrt{\Delta L})$. For any deterministic step-size rule $\{\eta_t\}_{t \geq 0}$ based on the observed history $\{(X_s,F(X_s),\nabla F(X_s))\}_{s \leq t}$, there exists $F: \br^{m\times n} \to \br$ satisfying Assumption~\ref{assumption:smooth} and $F(X_0)-F^\star \leq \Delta$ such that Algorithm~\ref{algorithm:normuon} with any fixed $\beta_1, \beta_2 \in [0,1)$, exact gradients, and $\ell=\delta=\alpha=0$ satisfies
\begin{equation*}
\|\nabla F(X_t)\|_\textnormal{nuc} > \epsilon \textnormal{ for all } 0 \leq t\leq T_0 = \left\lfloor\tfrac{m\Delta L}{32\epsilon^2}\right\rfloor.
\end{equation*}
\end{theorem}
Muon achieves a dimension-independent rate of $O(\Delta L\epsilon^{-2})$ under the same smoothness assumption and stationarity criterion~\citep{Shen-2026-Convergence}. Theorem~\ref{theorem:normuon-lower} therefore shows that row normalization can introduce an extra factor of $m$, even when the momentum parameters and the adaptive step-size rule are chosen. 

We next show that this dimension dependence is attainable. Indeed, we define
\begin{equation*} 
\kappa=\tfrac{(1-\delta)(1-\ell\min\{m,n\})\sqrt{1-\beta_2}}{(1+\delta)\sqrt{m}}, 
\end{equation*}
which quantifies the alignment retained after approximate polar computation and row normalization.
\begin{theorem} \label{theorem:normuon-upper-deterministic}
Suppose that Assumptions~\ref{assumption:smooth} and~\ref{assumption:polar} hold with $L>0$ and let $\ell<\frac{c}{\min\{m,n\}}$ where $c \in (0, 1)$ is a universal constant. Then, there exists some $T>0$ such that Algorithm~\ref{algorithm:normuon} with $\beta_1=0$, $\beta_2 \in [0,1)$, $\alpha \geq 0$, $\eta_t \equiv \eta := \frac{\kappa\epsilon}{0.2L\sqrt{mn}}$ and exact gradients satisfies $\min_{0\leq t<T} \|\nabla F(X_t)\|_\textnormal{nuc} \leq \epsilon$ and the total number of calls to the gradient oracle is bounded by
\begin{equation*}
O\left(\tfrac{m\Delta L}{\epsilon^2}\right), 
\end{equation*}
where $\Delta>0$ is an upper bound for the initial objective function gap $F(X_0)-F^\star$.
\end{theorem}
Theorem~\ref{theorem:normuon-upper-deterministic} shows that the upper bound for Algorithm~\ref{algorithm:normuon} matches the lower bound for exact polar computation with $\alpha=0$, and this bound is attained without first-moment momentum ($\beta_1=0$). The same $O(m\Delta L\epsilon^{-2})$ rate continues to hold under approximate polar computation, provided that $\ell < \frac{c}{\min\{m,n\}}$ for a universal constant $c \in (0,1)$.

\subsection{Stochastic setting}
We proceed to the stochastic setting, where Algorithm~\ref{algorithm:normuon} uses stochastic gradients. 
\begin{theorem} \label{theorem:normuon-upper-stochastic}
Suppose that Assumptions~\ref{assumption:smooth},~\ref{assumption:noise} and~\ref{assumption:polar} hold with $L>0$ and let $\ell<\frac{c}{\min\{m,n\}}$ where $c \in (0, 1)$ is a universal constant and $\epsilon \in (0,\sqrt{\Delta L})$. Then, there exists some $T>0$ such that Algorithm~\ref{algorithm:normuon} with $\beta_1=1-\min\{1,\frac{\kappa^2\epsilon^2}{64\sigma^2}\}$, $\beta_2 \in [0,1)$, $\alpha \geq 0$ and $\eta_t \equiv \eta := \frac{\kappa(1-\beta_1)\epsilon}{3.2L\sqrt{mn}}$ satisfies $\EE[\|\nabla F(X_{\widetilde{T}})\|_\textnormal{nuc}] \leq \epsilon$ and the total number of calls to the stochastic gradient oracle is bounded by
\begin{equation*}
O\left(\tfrac{m\Delta L}{\epsilon^2}+\tfrac{m^2\Delta L\sigma^2}{\epsilon^4}\right),
\end{equation*}
where $\Delta>0$ is an upper bound for the initial objective function gap $F(X_0)-F^\star$.
\end{theorem}
For fixed $\sigma>0$ and sufficiently small $\epsilon>0$, Muon has an upper bound of $O(\min\{m,n\}\Delta L\sigma^2\epsilon^{-4})$ which is shown to be optimal under the same smoothness and noise assumptions~\citep{Zhang-2026-Scale}. Theorem~\ref{theorem:normuon-upper-stochastic} shows a bound that contains an additional dimension-dependent factor of $m\max\{1,m/n\}$ for NorMuon in the stochastic setting. A promising direction is to either improve NorMuon's upper bound or prove that such an improvement is impossible by establishing a lower bound in the stochastic setting.

\subsection{Proof sketches}
\paragraph{Lower bound.} The proof is divided into (i) constructing a function $F(X)$ and (ii) proving it is hard for NorMuon within $T_0$ iterations. Indeed, we define $\lambda = \langle H_\star,D_\star\rangle = \frac{2m-1}{\sqrt{m(m^2+m-1)}}$, where $e_1^{(n)}$ is the first standard basis vector of $\br^n$ and
\begin{equation*}
H_\star=\tfrac{(m,1,\ldots,1)^\top}{\sqrt{m^2+m-1}}(e_1^{(n)})^\top \in \br^{m\times n},\quad D_\star=\tfrac{1}{\sqrt{m}}(1,1,\ldots,1)^\top (e_1^{(n)})^\top \in \br^{m\times n}.
\end{equation*} 
We construct two sequences $\{q_t\}_{t\leq T_0}$ and $\{f_t\}_{t\leq T_0}$, starting from $q_0=f_0=0$. For $0 \leq t<T_0$, we feed the history $\{(X_0+q_sD_\star,f_s,2\epsilon H_\star)\}_{s\leq t}$ to the deterministic step-size rule, which returns $\eta_t \geq 0$. We set $q_{t+1}:=q_t-\eta'_t$ with $\eta'_t:=0.2\sqrt{mn}\eta_t$ and $f_{t+1}:=f_t-\int_0^{\eta'_t}\max\{2\epsilon\lambda-L\min\{z,\eta'_t-z\},0\}dz$. Given $\{q_t\}_{t \leq T_0}$ and $P=H_\star-\lambda D_\star$, we define
\begin{equation*}
\begin{array}{lcl}
f(q) & = & \int_0^q\max_{0\leq j\leq T_0}\max\{2\epsilon\lambda-L|z-q_j|,0\}dz,\\
F(X) & = & f(\langle D_\star,X-X_0\rangle)+\tfrac{L}{2\|P\|_F^2}\Big(\langle P,X-X_0\rangle+\tfrac{2\epsilon\|P\|_F^2}{L}\Big)^2-\tfrac{2\epsilon^2\|P\|_F^2}{L}.
\end{array}
\end{equation*}
We show that Algorithm~\ref{algorithm:normuon} applied to $F$ observes exactly  $\{(X_0+q_sD_\star,f_s,2\epsilon H_\star)\}_{s\leq T_0}$, that is,
\begin{equation*}
X_t=X_0+q_tD_\star,\quad F(X_t)=f_t,\quad \nabla F(X_t)=2\epsilon H_\star,\quad\text{for all } 0\leq t\leq T_0,
\end{equation*}
which implies $\|\nabla F(X_t)\|_\textnormal{nuc}=2\epsilon>\epsilon$ for all $t\leq T_0$. Indeed, we have $F(X_0+q_tD_\star)=f_t$ and $\nabla F(X_0+q_tD_\star)=2\epsilon\lambda D_\star+2\epsilon P=2\epsilon H_\star$ for all $t\leq T_0$. It remains to show that $X_t=X_0+q_tD_\star$, which we prove by induction on $t$. The case $t=0$ is trivial. Suppose that $X_s=X_0+q_sD_\star$ for all $s\leq t$. Then $G_s=2\epsilon H_\star$ for all $s\leq t$, and hence $M_{t+1}$ is a positive multiple of $H_\star$ and $O_{t+1}=H_\star$. Since every row of $H_\star$ is a positive multiple of $e_1^{(n)}$, row normalization equalizes the rows and gives $D_{t+1}=0.2\sqrt{mn}D_\star$ for any fixed $\beta_1,\beta_2 \in [0,1)$. Therefore, the step-size rule observes $\{(X_0+q_sD_\star,f_s,2\epsilon H_\star)\}_{s\leq t}$ and returns the same $\eta_t$, which gives $X_{t+1}=X_0+q_{t+1}D_\star$.

Since $F(X_0)-F^\star \leq \frac{2\epsilon^2}{L}+\frac{4T_0\epsilon^2\lambda^2}{L}$, $\lambda^2<\frac{4}{m}$, and $\epsilon<\frac14\sqrt{\Delta L}$, we have $F(X_0)-F^\star\leq\Delta$ whenever $T_0\leq\frac{m\Delta L}{32\epsilon^2}$. Hence, our construction allows $T_0=\Omega(m\Delta L\epsilon^{-2})$.

\paragraph{Upper bound.} The key is to apply the descent inequality with $X=X_t$ and $Y=X_{t+1}$ as follows: 
\begin{equation} \label{eq:main-text-descent}
F(X_{t+1})-F(X_t) \leq -\eta_t\langle\nabla F(X_t),D_{t+1}\rangle+\tfrac{L\eta_t^2}{2}\|D_{t+1}\|_\textnormal{op}^2.
\end{equation}
By definition, $\|D_{t+1}\|_\textnormal{op} \leq 0.2\sqrt{mn}$. It suffices to bound $\langle\nabla F(X_t),D_{t+1}\rangle$ from below, and this is where row normalization matters. Lemma~\ref{lem:direction} shows that
\begin{equation} \label{eq:main-text-alignment}
\langle M_{t+1},D_{t+1}\rangle\geq 0.2\kappa\sqrt{mn}\|M_{t+1}\|_\textnormal{nuc}
\end{equation}
for any $\alpha \geq 0$ under approximate polar computation, where $\kappa=\Theta(m^{-1/2})$.

In the deterministic setting with $\beta_1=0$, we have $M_{t+1}=\nabla F(X_t)$. With $\eta$ defined in Theorem~\ref{theorem:normuon-upper-deterministic}, Eq.~\eqref{eq:main-text-descent} becomes $F(X_{t+1})-F(X_t)\leq-\tfrac{\kappa^2\epsilon}{L}\|\nabla F(X_t)\|_\textnormal{nuc}+\frac{\kappa^2\epsilon^2}{2L}$, which guarantees a decrease of at least $\frac{\kappa^2\epsilon^2}{2L}$ whenever $\|\nabla F(X_t)\|_\textnormal{nuc}>\epsilon$. Telescoping Eq.~\eqref{eq:main-text-descent} yields the desired result. 

In the stochastic setting, we decompose $M_{t+1}=\nabla F(X_t)+E_t+B_t$ using a drift term $E_t$ and a noise term $B_t$. Using Eq.~\eqref{eq:main-text-alignment} and the definition of $\kappa$, $\langle\nabla F(X_t),D_{t+1}\rangle\geq 0.2\sqrt{mn}(\kappa\|\nabla F(X_t)\|_\textnormal{nuc}-2\|E_t\|_\textnormal{nuc}-2\|B_t\|_F)$. Here, $\|E_t\|_\textnormal{nuc}$ is controlled by Assumption~\ref{assumption:smooth}, and $\|B_t\|_F$ is controlled by the martingale structure of the gradient noise and Assumption~\ref{assumption:noise}. Choosing $\beta_1$ and $\eta$ and telescoping Eq.~\eqref{eq:main-text-descent} yields the desired result.

\section{Experiments}\label{sec:exp}
We compare Muon and NorMuon on synthetic least-squares problems motivated by Theorem~\ref{theorem:normuon-lower}, image classification with a convolutional neural network (CNN), and LLM pretraining. We run the synthetic experiments on an AMD Ryzen 7 8845HS CPU, and the CNN and LLM experiments on 1 and 8 NVIDIA H200 GPUs, respectively. In all neural network experiments, we use the practical variant of NorMuon~\citep{Li-2026-Normuon}, which replaces the update direction $D_{t+1}=\frac{0.2\sqrt{mn}}{\|\overline{O}_{t+1}\|_F}\overline{O}_{t+1}$ in Algorithm~\ref{algorithm:normuon} with
\begin{equation} \label{eq:normuon-practical}
D_{t+1}=\tfrac{\|O_{t+1}\|_F}{\|\overline{O}_{t+1}\|_F}\overline{O}_{t+1},\quad X_{t+1}\gets X_t-\eta_tD_{t+1},
\end{equation}
so that the update has the same Frobenius norm as the Muon update, which simplifies parameter tuning. Since this change only rescales the step size, our lower bound still applies, and our upper bounds hold once the rescaling is absorbed into the step size.

\subsection{Synthetic experiment}\label{sec:exp-synthetic}
The hard function in Theorem~\ref{theorem:normuon-lower} depends on the step-size rule, and hence it cannot serve as a fixed benchmark for candidate methods. We instead minimize the following least-squares objective with imbalanced rows: 
\begin{equation*}
F(X) = \tfrac{1}{2}\|X-\|C\|_{F}^{-1}C\|_F^2, \textnormal{ where } C \in \br^{m\times n},\ [C]_{i,:}=\tfrac{1}{i}u_i^\top,
\end{equation*}
and $u_1,\ldots,u_m$ are drawn independently and uniformly from the unit sphere in $\mathbb{R}^n$. We set $n=8$ and $m\in\{8,32,128,512,2048\}$. We run Muon, NorMuon (Algorithm~\ref{algorithm:normuon}), and NorMuon with Eq.~\eqref{eq:normuon-practical} from $X_0=M_0=v_0=0$ with $\alpha=0$ for $T=1000$ iterations, and tune a constant step size and the momentum parameters for each method and each $m$. See Appendix~\ref{app:exp} for more details.

We consider the reduction $\log_{10}(F(X_0))-\log_{10}(\min_{0\leq t\leq T} F(X_t))$, where a larger value is better. We tune the hyperparameters on 50 random seeds and report the results over 100 other random seeds. We compare Muon and NorMuon on the hard instances and report the results in Table~\ref{tab:synthetic-iterations}. As $m$ grows from $8$ to $2048$, the loss reduction of NorMuon drops, while that of Muon remains nearly unchanged (see Appendix~\ref{app:exp} for the loss curves). The two versions of NorMuon perform almost identically. Since NorMuon with Eq.~\eqref{eq:normuon-practical} has the same update norm as Muon, the gap can be attributed to row normalization.
\begin{table}[!ht]
\centering
\caption{Mean $\pm$ twice the standard error of the loss reduction $(\log_{10}{F(X_0)}-\log_{10}{\min_{0\leq t\leq T} F(X_t)})$ over 100 random seeds. A larger value means better performance.}
\label{tab:synthetic-iterations}
\begin{tabular}{lrrrrr}
\toprule
$m$ & $8$ & $32$ & $128$ & $512$ & $2048$ \\
\midrule
Muon & $5.70_{\pm 0.03}$ & $5.70_{\pm 0.02}$ & $5.70_{\pm 0.02}$ & $5.72_{\pm 0.03}$ & $5.72_{\pm 0.03}$ \\
NorMuon (Algorithm~\ref{algorithm:normuon}) & $5.70_{\pm 0.03}$ & $5.46_{\pm 0.02}$ & $5.00_{\pm 0.02}$ & $4.12_{\pm 0.01}$ & $3.15_{\pm 0.01}$ \\
NorMuon with Eq.~\eqref{eq:normuon-practical} & $5.70_{\pm 0.03}$ & $5.46_{\pm 0.02}$ & $5.00_{\pm 0.01}$ & $4.13_{\pm 0.01}$ & $3.15_{\pm 0.01}$ \\
\bottomrule
\end{tabular}
\end{table}

\subsection{Image classification}\label{sec:exp-cifar}
\paragraph{Training.} We train CIFARNET~\citep{Jordan-2024-94}, a CNN with 2M parameters, on CIFAR-10~\citep{Krizhevsky-2009-Learning}. CIFARNET has $3$ groups of convolutional layers with widths $64$, $256$, and $256$. Each group applies a $3 \times 3$ convolution, max pooling, BatchNorm, GELU, a second $3 \times 3$ convolution, BatchNorm, and GELU. Flattening each convolution into a matrix in $\br^{d_\textnormal{out} \times 9d_\textnormal{in}}$, where $d_\textnormal{in}$ and $d_\textnormal{out}$ are its numbers of input and output channels, gives six matrices: one each of sizes $64 \times 216$, $64 \times 576$, and $256 \times 576$, and three of size $256 \times 2304$. We update these matrix-valued parameters with AdamW~\citep{Kingma-2015-Adam, Loshchilov-2019-Decoupled}, Muon, or NorMuon with row or column normalization, and all other parameters with the same SGD optimizer, following~\citet{Jordan-2024-Muon}. We train for 50 epochs with a batch size of 512. See Appendix~\ref{app:exp} for tuning details.

\paragraph{Evaluation.} We select hyperparameters using the mean final validation accuracy over 20 seeds on a fixed 45,000/5,000 training/validation split of the training set. We retrain each selected configuration on all 50,000 training images with 50 new seeds and evaluate it on the test set.

\paragraph{Results.} We compare AdamW, Muon, and both NorMuon variants in Table~\ref{tab:cifarnet}. All three Muon-family optimizers outperform AdamW. NorMuon with column normalization, which normalizes along the larger dimension of every weight matrix, attains the highest mean test accuracy by more than a standard error, while Muon attains the lowest mean test loss.
\begin{table}[!ht]
\centering
\caption{Mean $\pm$ standard error of the test loss and accuracy over 50 seeds for CIFARNET.}
\label{tab:cifarnet}
\begin{tabular}{lcccc}
\toprule
& AdamW & Muon & NorMuon (column) & NorMuon (row) \\ \midrule
Test loss & $0.4298_{\pm 0.0005}$ & $\mathbf{0.3884}_{\pm 0.0004}$ & $0.3889_{ \pm 0.0005}$ & $0.3900_{ \pm 0.0004}$ \\
Test accuracy & $92.33_{\pm 0.02}\%$ & $93.79_{\pm 0.03}\%$ & $\mathbf{93.87}_{\pm 0.02}\%$ & $93.83_{\pm 0.02}\%$ \\
\bottomrule
\end{tabular}
\end{table}

\subsection{LLM pretraining}\label{sec:exp-llm}
\paragraph{Training.} We pretrain decoder-only transformers with $286$M, $1.38$B, and $6.44$B total parameters using the nanochat codebase~\citep{Karpathy-2025-Nanochat} on the Nemotron-CLIMB dataset~\citep{Diao-2025-Nemotron} for $2.20$B, $14.6$B, and $19.7$B tokens, respectively. Following the practice for Muon~\citep{Jordan-2024-Muon}, Muon or NorMuon trains only the transformer matrices, which have about $84.9$M, $679$M, and $1.61$B parameters, and AdamW trains all other parameters. The token budget of the $6.44$B model follows the compute-optimal setting of nanochat, which allocates about $10.5$ tokens per parameter in the transformer matrices and the language model head. The $286$M and $1.38$B models use $20$ tokens per parameter in the transformer matrices and the language model head, which exceeds this setting. See Appendix~\ref{app:exp} for more details. 

We compare NorMuon with Muon and include AdamW with default hyperparameters, i.e., step size $0.001$ and momentum factors $(0.9,0.999)$, as a reference for the $1.38$B and $6.44$B models. As in~\citet{Karpathy-2025-Nanochat}, the query, key, value, and output projections are separate square matrices. NorMuon normalizes these square matrices along their rows and the MLP matrices along their larger dimension. The decoupled weight decay is $0.02$ throughout. For Muon and NorMuon, the step size is set to $\eta\sqrt{2^{-19}B}$~\citep{Krizhevsky-2014-One}, where $\eta$ is the base step size and $B$ is the batch size in tokens. We select $\eta=0.03$ and $\eta=0.01$ by a grid search on the $286$M and $1.38$B models, respectively (Section~\ref{sec:exp-ablation}) and reuse $\eta=0.01$ for the $6.44$B model. Both optimizers use Nesterov-type momentum~\citep{Dozat-2016-Incorporating} with $\beta_1=0.95$ and $5$ PolarExpress iterations~\citep{Amsel-2026-Polar}, and NorMuon uses $\beta_2=0.95$. We perform a grid search to select the best step size for AdamW on the $286$M model from $\{3\times 10^{-4}, 10^{-3}, 3\times 10^{-3}\}$ and pick $10^{-3}$.

\paragraph{Evaluation.} We hold out $41.94$M tokens to compute the validation loss and $167.77$M tokens to compute the test loss of the final model. Following OLMES~\citep{Gu-2025-OLMES}, we report accuracy on $10$ knowledge and reasoning benchmarks: MMLU~\citep{Hendrycks-2021-Measuring}, HellaSwag (HSwag)~\citep{Zellers-2019-Hellaswag}, PIQA~\citep{Bisk-2020-Piqa}, WinoGrande~\citep{Sakaguchi-2021-Winogrande}, ARC-Challenge (ARC-C) and ARC-Easy (ARC-E)~\citep{Clark-2018-Think}, BoolQ~\citep{Clark-2019-BoolQ}, CommonsenseQA (CSQA)~\citep{Talmor-2019-CommonsenseQA}, Social IQa (SIQA)~\citep{Sap-2019-Social}, and OpenBookQA (OBQA)~\citep{Mihaylov-2018-Can}.

\paragraph{Results.} We show that NorMuon achieves lower validation loss than Muon for most of training at all $3$ model sizes in Figure~\ref{fig:llm_results}. We show that NorMuon improves upon Muon's test loss and average downstream accuracy for the $1.38$B and $6.44$B models in Table~\ref{tab:llm_results}. For the $1.38$B model, we run Muon and NorMuon with $3$ random seeds and report the mean. Standard deviations are reported in Appendix~\ref{app:exp}. All other runs use a single seed. Both Muon and NorMuon outperform AdamW. The $6.44$B model uses a different tokenizer, and its loss is not directly comparable to those of the other models.
\begin{table}[ht!]
\caption{LLM pretraining downstream accuracies and test losses.}
\label{tab:llm_results}
\centering
\setlength{\tabcolsep}{1pt}
\begin{tabular}{@{}ll*{12}{c}@{}}
\toprule
& & MMLU & HSwag & PIQA & Wino. & ARC-C & ARC-E
& BoolQ & CSQA & SIQA & OBQA & Avg & Loss \\
\midrule
\multirow{3}{*}{\rotatebox{90}{$1.38$B}}
& AdamW & $31.9$ & $54.9$ & $74.6$ & $54.2$ & $45.1$ & $72.3$ & $57.9$ & $56.6$ & $46.2$ & $48.2$ & $54.2$ & $2.3553$ \\
& Muon & $\mathbf{33.3}$ & $60.3$ & $75.9$ & $55.5$ & $43.8$ & $73.6$ & $\mathbf{64.4}$ & $\mathbf{60.0}$ & $49.1$ & $47.1$ & $56.3$ & $2.2914$ \\
& NorMuon & $\mathbf{33.3}$ & $\mathbf{60.9}$ & $\mathbf{76.1}$ & $\mathbf{58.4}$ & $\mathbf{45.4}$ & $\mathbf{75.0}$ & $62.0$ & $59.6$ & $\mathbf{50.2}$ & $\mathbf{49.8}$ & $\mathbf{57.1}$ & $\mathbf{2.2810}$ \\
\midrule
\multirow{3}{*}{\rotatebox{90}{$6.44$B}}
& AdamW & $34.2$ & $60.8$ & $76.4$ & $57.2$ & $43.5$ & $73.1$ & $66.9$ & $60.2$ & $48.8$ & $46.4$ & $56.8$ & $2.3748$ \\
& Muon & $36.2$ & $67.3$ & $\mathbf{78.5}$ & $\mathbf{61.4}$ & $\mathbf{48.7}$ & $\mathbf{79.6}$ & $65.7$ & $63.1$ & $\mathbf{53.0}$ & $50.4$ & $60.4$ & $2.2803$ \\
& NorMuon
& $\mathbf{36.4}$ & $\mathbf{68.2}$ & $78.0$ & $60.5$ & $47.9$ & $76.5$ & $\mathbf{72.8}$ & $\mathbf{64.5}$ & $52.5$ & $\mathbf{51.6}$ & $\mathbf{60.9}$ & $\mathbf{2.2674}$ \\
\bottomrule
\end{tabular}
\end{table}
\begin{figure}[!ht]
\centering
\includegraphics[width=\linewidth]{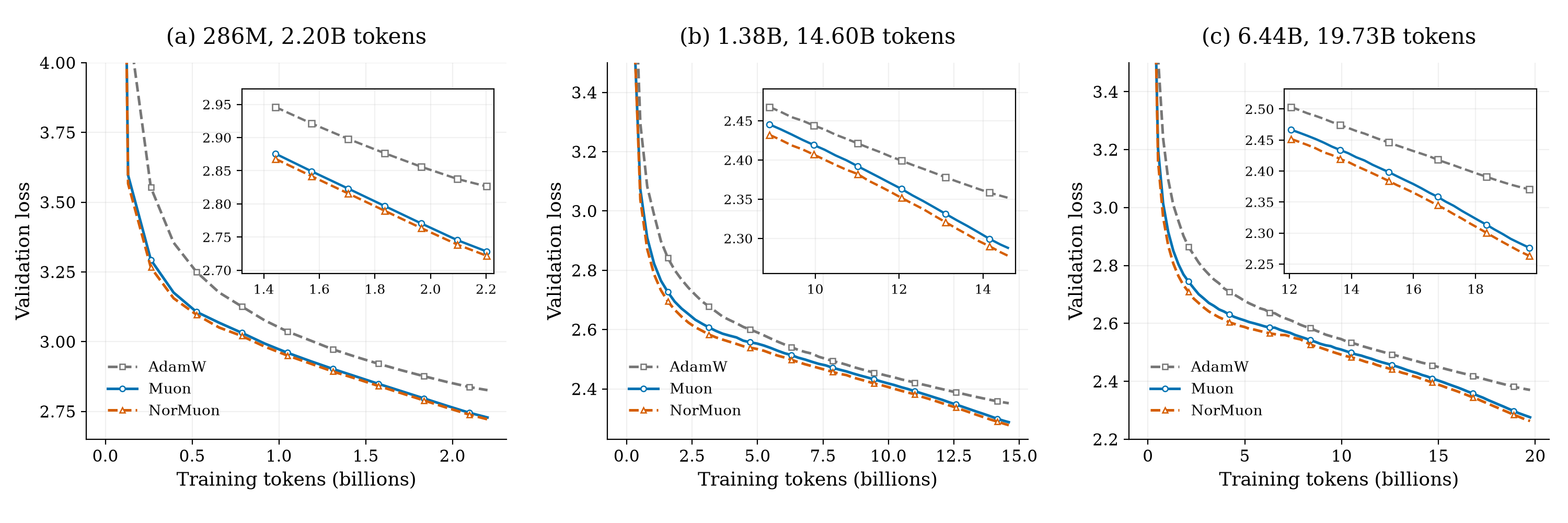}
\caption{Validation loss in LLM pretraining for the (a) $286$M, (b) $1.38$B, and (c) $6.44$B models.} \label{fig:llm_results}
\end{figure}

\subsection{Ablation studies}\label{sec:exp-ablation}
\paragraph{Step-size sensitivity.} We test how sensitive Muon and NorMuon are to step sizes. We train the $286$M model on $2.20$B tokens with step sizes in $\{0.01,0.02,0.03,0.04\}$ and the $1.38$B model on $14.6$B tokens with step sizes in $\{0.005,0.01,0.02\}$, and compare NorMuon with Muon at each step size. Figure~\ref{fig:ablation}(a) shows the final validation loss. NorMuon has lower loss than Muon at every step size, and both reach their lowest loss at $0.03$ for the $286$M model and $0.01$ for the $1.38$B model. 
\begin{figure}[!ht]
\centering
\includegraphics[width=\linewidth]{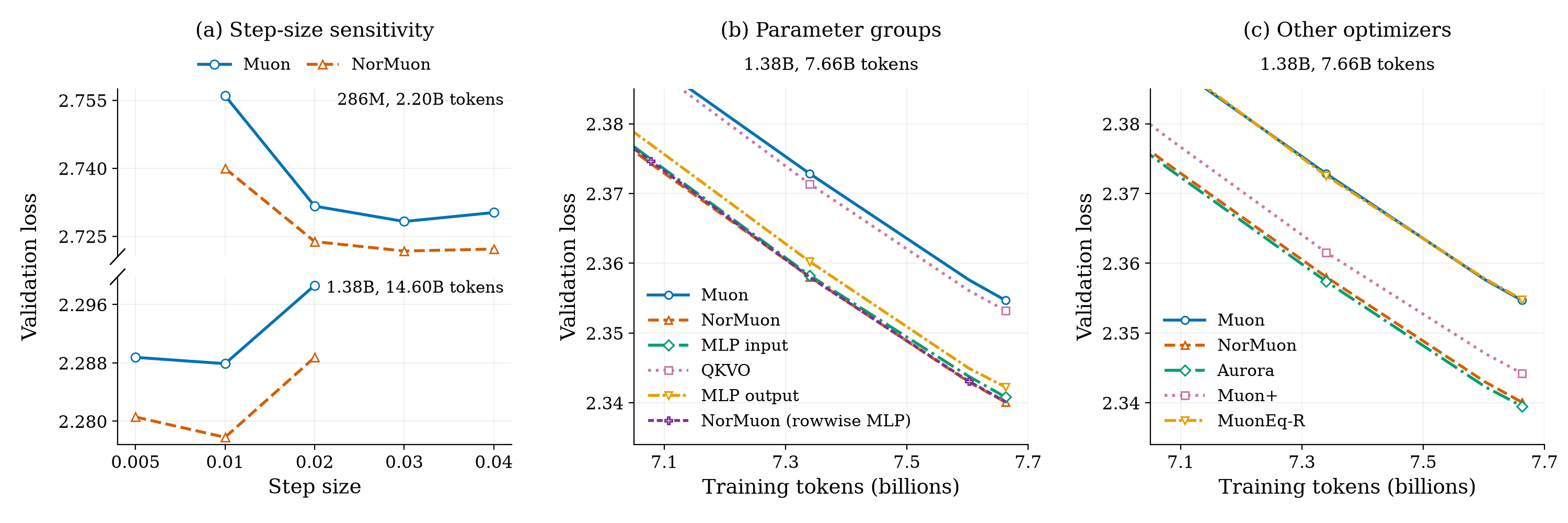}
\caption{Ablation studies. (a) Final validation loss versus step size. (b) NorMuon applied to different groups of matrices. (c) Comparison with other Muon variants that use normalization.} \label{fig:ablation}
\end{figure}

\noindent In what follows, we instead train the $1.38$B model in the compute-optimal setting of $7.66$B tokens.
\paragraph{Parameter groups.} We test which weight matrices benefit from normalization. We apply NorMuon only to the attention query/key/value/output projections (QKVO), the MLP input projections, or the MLP output projections, with Muon optimizing the remaining matrices. We compare these variants with Muon, NorMuon on all transformer matrices, and NorMuon (row-wise MLP), which normalizes the MLP matrices along their rows instead of their larger dimension. As shown in Figure~\ref{fig:ablation}(b) and the left half of Table~\ref{tab:ablation_test}, normalizing only the QKVO matrices gives nearly the same validation and test losses as Muon, while normalizing only the MLP input or only the MLP output projections recovers most of the gain of NorMuon, and NorMuon (row-wise MLP) nearly matches NorMuon. Thus, the gain of NorMuon comes mainly from the MLP matrices. This aligns with the findings of~\citet{Dewulf-2026-Aurora}. 
 
\paragraph{Other matrix optimizers with normalization.} We compare NorMuon with other Muon variants that use normalization. The baselines include Muon, Aurora~\citep{Dewulf-2026-Aurora}, Muon+~\citep{Zhang-2026-Muon}, and MuonEq-R~\citep{Chang-2026-Muoneq}. As shown in Figure~\ref{fig:ablation}(c) and the right half of Table~\ref{tab:ablation_test}, Aurora and NorMuon achieve nearly identical validation and test losses, with Aurora's losses being slightly lower, followed by Muon+, while MuonEq-R performs nearly identically to Muon. With the same step size and 5 PolarExpress iterations, Aurora, NorMuon, and Muon+, which all normalize the orthogonalized update, improve on Muon.
\begin{table}[!ht]
\centering
\caption{Test loss for parameter group and optimizer ablations.}
\label{tab:ablation_test}
\begin{tabular}{@{}lrr@{\hspace{2em}}lrr@{}} \toprule
NorMuon parameter groups & $286$M & $1.38$B & Other optimizers & $286$M & $1.38$B \\ \midrule
Muon & $2.7311$ & $2.3587$ & Muon & $2.7311$ & $2.3587$ \\
NorMuon & $\mathbf{2.7248}$ & $\mathbf{2.3442}$ & NorMuon & $2.7248$ & $2.3442$ \\
QKVO only & $2.7319$ & $2.3572$ & Aurora & $\mathbf{2.7237}$ & $\mathbf{2.3435}$ \\
MLP input only & $2.7258$ & $2.3448$ & Muon+ & $2.7265$ & $2.3483$ \\
MLP output only & $2.7260$ & $2.3463$ & MuonEq-R & $2.7305$ & $2.3587$ \\
NorMuon (row-wise MLP) & $2.7260$ & $\mathbf{2.3442}$ & & & \\ \bottomrule
\end{tabular}
\end{table}

\section{Conclusion}\label{sec:conclusion}
We studied the effect of row normalization on Muon through the worst-case convergence of NorMuon in operator-norm geometry. Indeed, we established matching deterministic lower and upper bounds of $\Theta(mL\epsilon^{-2})$, revealing an extra factor of $m$ compared with Muon. The lower bound holds with exact polar computation, any fixed momentum parameters and any deterministic adaptive step sizes. We also extended the upper bound analysis to stochastic settings, allowing approximate polar computation in both upper bound analyses. Synthetic experiments show a slowdown for NorMuon, whereas LLM pretraining experiments demonstrate improved performance over Muon. These findings highlight that balanced neuron-wise updates alone do not guarantee faster optimization in the worst case. Future directions include identifying structural properties of LLM training that explain when row normalization is beneficial and developing a theory that captures these benefits.

\section*{Acknowledgments}
We sincerely appreciate Buzz High Performance Computing (\hyperlink{https://www.buzzhpc.ai}{\texttt{https://www.buzzhpc.ai}}, \texttt{info@buzzhpc.ai}) for providing computational resources and support for this work. The second author is partially supported by a start-up fund and the Early Career Scholarship Support Grant at Columbia University.

\bibliographystyle{plainnat}
\bibliography{ref}

\newpage \appendix
\section{Further Related Works} \label{app:further-related-work}
We make some comments on other topics, including more discussions on matrix-aware optimization methods, the theoretical analysis of matrix-aware optimization methods, lower bound analysis, and other optimizers and analysis. For an overview of neural network optimization methods and analysis, we refer to the monographs~\citep{Nesterov-2018-Lectures, Zhang-2023-Dive}.

\paragraph{More discussion on matrix-aware optimization methods.} Earlier matrix optimizers exploit the structure of weight matrices through Kronecker-factored or second-moment preconditioners~\citep{Martens-2015-Optimizing, Gupta-2018-Shampoo, Vyas-2025-SOAP}. Muon~\citep{Jordan-2024-Muon} and Scion~\citep{Pethick-2025-Training} depart from preconditioning and instead construct updates by dualizing the gradient under the operator norm. Several other Muon variants modify how the update is scaled or preconditioned. AdaMuon~\citep{Si-2025-Adamuon} applies element-wise second-moment normalization to the orthogonalized update. NAMO~\citep{Zhang-2026-Adam} scales the orthogonalized momentum by a single adaptive step size, while its diagonal extension NAMO-D applies column-wise adaptive scaling. Muon$^2$~\citep{Liu-2026-Muon} instead preconditions the momentum before orthogonalization, which improves the conditioning of the input to the polar computation. Muown~\citep{Lion-2026-Muown, Hubler-2026-Muown} controls row magnitudes through reparameterization. Stateless normalization has also been studied: SWAN~\citep{Ma-2025-Swan} applies row normalization followed by orthogonalization, and SinkGD~\citep{Scetbon-2025-Gradient} replaces the orthogonalization with alternating row and column normalization. Other extensions target distributed efficiency~\citep{Ahn-2025-Dion, Amsel-2026-Dion3}, step-size scaling based on polar decomposition~\citep{Lau-2025-Polargrad}, and more accurate polar approximation~\citep{Amsel-2026-Polar}.

\paragraph{Analysis of matrix-aware optimization methods.} A conceptual basis for Muon is that the gradient is a dual vector and should be mapped back to the primal space before being used as an update. \citet{Bernstein-2024-Old} identify the orthogonalized update as steepest descent under the spectral norm, and modular dualization~\citep{Bernstein-2025-Modular, Large-2024-Scalable} extends this view by assigning each layer an operator norm according to its input-output structure; see also~\citep{Pethick-2025-Training, Chen-2026-Muon}. Within this operator-norm geometry, exact polar analyses establish a dimension-free $O(\Delta L\epsilon^{-2})$ rate for Muon~\citep{Shen-2026-Convergence}, and Muon with weight decay converges as a stochastic Frank-Wolfe method~\citep{Sfyraki-2026-Lions}. \citet{Kim-2026-Convergence} prove that finite Newton-Schulz steps only introduce a multiplicative factor that approaches one doubly exponentially, and \citet{Choudhury-2026-Muon} further handle Nesterov momentum and heavy-tailed noise. Other works study implicit and simplicity biases, local quadratic models, structured problems, and nonsmooth settings~\citep{Fan-2026-Implicit, Dragutinovic-2026-To, Davis-2025-Spectral, Ma-2026-Preconditioning, Gonon-2026-Insights, Parshakova-2026-Muon}.

\paragraph{Lower bound analysis.} Classical first-order complexity analysis applies to every gradient-based optimization method and therefore gives no specific treatment to individual update rules \citep{Nemirovski-1983-Problem,Carmon-2020-Lower,Carmon-2021-Lower,Arjevani-2023-Lower}, leaving open the best rate that a particular algorithm can achieve. Algorithm-dependent convergence analysis instead analyzes a fixed update rule and seeks the best convergence result attainable by tuning it. Additional examples include tight stationary-point lower bounds for SGD \citep{Drori-2020-Complexity}, explicit dimension-dependent lower bounds for cyclic coordinate descent relative to randomized coordinate descent \citep{Sun-2021-Worst}, lower bounds distinguishing random-shuffling and incremental-gradient schemes \citep{Safran-2020-Good}, and exact worst-case constructions for gradient descent with exact line search~\cite{De-2017-Worst}. More recently, related lower-bound analyses have been developed for sign-based methods~\citep{Tao-2026-And}.

\paragraph{Other optimizers and analysis.}  Classical optimizers treat weight matrices as vectors, from SGD with momentum~\citep{Sutskever-2013-Importance} to adaptive methods that rescale each coordinate by accumulated or exponentially averaged squared gradients~\citep{Duchi-2011-Adaptive, Tieleman-2012-Lecture, Kingma-2015-Adam, Loshchilov-2019-Decoupled}. Other works rescale updates at a coarser granularity to exploit parameter structure~\citep{You-2020-Large}. Adafactor~\citep{Shazeer-2018-Adafactor} factorizes the second moment of a weight matrix into row and column statistics to save memory. Adam-mini~\citep{Zhang-2025-Adam} uses a single second-moment estimate for each Hessian-informed block. NorMuon adopts this normalization, but applies it to the row-wise norms of the orthogonalized update rather than to gradient coordinates.

Our analysis is also motivated by prior works that study adaptive gradient methods. \citet{Wilson-2017-Marginal} construct a linear classification problem on which adaptive methods converge to poorly generalizing solutions, whereas SGD finds the minimum-norm solution. \citet{Reddi-2018-Convergence} construct simple convex problems on which Adam fails to converge. \citet{Zhang-2022-Adam} show that Adam converges when $\beta_2$ is chosen large enough after the problem is fixed. In contrast, we show that the worst-case cost of row normalization in NorMuon is not non-convergence but an extra dimension-dependent factor in the iteration complexity, and this cost persists under any fixed momentum parameters.

\section{Missing Proofs} \label{app:proofs}
Throughout this section, we let $d=\min\{m,n\}$ and $e_i^{(n)} \in \br^n$ be the $i^\textnormal{th}$ basis vector of $\br^n$. 
\subsection{Proof of Theorem~\ref{theorem:normuon-lower}}
We define $\lambda=\langle H_\star,D_\star\rangle=\tfrac{2m-1}{\sqrt{m(m^2+m-1)}}$, where
\begin{equation}
\label{eq:def-H-D}
H_\star:=\tfrac{(m,1,\ldots,1)^\top}{\sqrt{m^2+m-1}}(e_1^{(n)})^\top \in \br^{m\times n},\quad D_\star:=\tfrac{1}{\sqrt{m}}(1,1,\ldots,1)^\top (e_1^{(n)})^\top \in \br^{m\times n}.
\end{equation}
We construct two sequences $\{q_t\}_{t=0}^{T_0}$ and $\{f_t\}_{t=0}^{T_0}$ by setting $q_0=0$ and $f_0=0$. For $0\leq t<T_0$, we feed the history $\{(X_0+q_sD_\star, f_s, 2\epsilon H_\star)\}_{s\leq t}$ to the deterministic adaptive step-size rule, which returns a fixed $\eta_t \geq 0$. We let $\eta'_t=0.2\sqrt{mn}\eta_t$ and define
\begin{equation}\label{eq:lower-transcript}
q_{t+1}=q_t-\eta'_t,\quad f_{t+1}=f_t-\int_0^{\eta'_t} \max\{2\epsilon\lambda-L\min\{z,\eta'_t-z\},0\}dz.
\end{equation}
We also define 
\begin{align}
f(q) & = \ \int_0^q \max_{0 \leq j \leq T_0}\max\{2\epsilon\lambda-L|z-q_j|,0\}dz, \label{eq:lower-interpolant} \\
F(X) & = \ f(\langle D_\star,X-X_0\rangle)+\tfrac{L}{2\|H_\star-\lambda D_\star\|_F^2}(\langle H_\star-\lambda D_\star,X-X_0\rangle+\tfrac{2\epsilon\|H_\star-\lambda D_\star\|_F^2}{L})^2\notag\\
&\quad -\tfrac{2\epsilon^2\|H_\star-\lambda D_\star\|_F^2}{L}.\label{eq:lower-objective}
\end{align}
We claim that, when applied to $F$ with any fixed $\beta_1,\beta_2\in[0,1)$, exact gradients and $\ell=\delta=\alpha=0$, Algorithm~\ref{algorithm:normuon} satisfies
\begin{equation} \label{eq:lower-bound-claim}
X_t=X_0+q_tD_\star,\quad F(X_t)=f(q_t)=f_t,\quad \nabla F(X_t)=2\epsilon H_\star,\quad \textnormal{for all } t=0, 1, \ldots, T_0.
\end{equation}
In other words, Algorithm~\ref{algorithm:normuon} observes $\{(X_0+q_sD_\star, f_s, 2\epsilon H_\star)\}_{s\leq T_0}$ and selects the same $\{\eta_s\}_{s<T_0}$. Since $\|2\epsilon H_\star\|_\textnormal{nuc}=2\epsilon>\epsilon$, we obtain the desired result in Theorem~\ref{theorem:normuon-lower}. Thus, it suffices to show that (i) Eq.~\eqref{eq:lower-bound-claim} holds; (ii) $F$ satisfies Assumption~\ref{assumption:smooth}; and (iii) $F(X_0)-F^\star \leq \Delta$.

We prove (i). The sequence $\{q_t\}_{t \leq T_0}$ is nonincreasing, which implies that the distance from any $q \in [q_{t+1},q_t]$ to $\{q_0,\ldots,q_{T_0}\}$ is $\min\{q-q_{t+1},q_t-q\}$. Therefore, we have
\begin{equation*}
f(q_t)-f(q_{t+1}) = \int_0^{\eta'_t} \max\{2\epsilon\lambda-L\min\{z,\eta'_t-z\},0\}dz = f_t-f_{t+1}. 
\end{equation*}
Since $f(q_0)=f_0=0$, we have $f(q_t)=f_t$ for all $0 \leq t \leq T_0$. By definition, we have
\begin{equation*}
\nabla F(X)=f'(\langle D_\star,X-X_0\rangle)D_\star + (\tfrac{L\langle H_\star-\lambda D_\star,X-X_0\rangle}{\|H_\star-\lambda D_\star\|_F^2}+2\epsilon)(H_\star-\lambda D_\star). 
\end{equation*}
Since $j=t$ attains the maximum in Eq.~\eqref{eq:lower-interpolant}, we have $f'(q_t)=2\epsilon\lambda$. Putting these pieces together with $\langle H_\star-\lambda D_\star,D_\star\rangle=0$ yields that $F(X_0+q_tD_\star)=f_t$ and $\nabla F(X_0+q_tD_\star)=2\epsilon H_\star$ for all $0\leq t\leq T_0$.

It remains to show $X_t=X_0+q_tD_\star$ for all $t \leq T_0$. We prove this by induction on $t$. The case of $t=0$ follows from $q_0=0$. Suppose that $X_s=X_0+q_sD_\star$ for all $s\leq t<T_0$. Then, we have $G_s=\nabla F(X_s)=2\epsilon H_\star$ and $F(X_s)=f_s$ for $s\leq t$. Thus, the deterministic adaptive step-size rule observes $\{(X_0+q_\tau D_\star,f_\tau,2\epsilon H_\star)\}_{\tau\leq t}$ and selects $\eta_t$. Since $\operatorname{rank}(H_\star)=\|H_\star\|_F=1$, Algorithm~\ref{algorithm:normuon} gives $M_{s+1}=2\epsilon(1-\beta_1^{s+1})H_\star$ and $O_{s+1}=\operatorname{Polar}(H_\star)=H_\star$ for all $s \leq t$. By the definition of $H_\star$ (see Eq.~\eqref{eq:def-H-D}), we have
\begin{equation} \label{eq:lower-bound-v}
v_{t+1} = \sum_{s=0}^t \beta_2^{t-s}\tfrac{1-\beta_2}{n}\operatorname{diag}(O_{s+1}O_{s+1}^\top) = \tfrac{1-\beta_2^{t+1}}{n(m^2+m-1)}(m^2,1,\ldots,1)^\top.
\end{equation}
Since $\alpha=0$, we have
\begin{equation*}
[\overline{O}_{t+1}]_{i,:} = \tfrac{1}{\sqrt{[v_{t+1}]_i}} [O_{t+1}]_{i,:} = \tfrac{1}{\sqrt{[v_{t+1}]_i}} [H_\star]_{i,:} = \begin{cases} \tfrac{m}{\sqrt{(m^2+m-1) [v_{t+1}]_1}} e_1^{(n)}, & i = 1, \\ \tfrac{1}{\sqrt{(m^2+m-1) [v_{t+1}]_i}} e_1^{(n)}, & i = 2,\ldots,m. \end{cases}
\end{equation*}
Together with Eq.~\eqref{eq:lower-bound-v}, this implies that $[\overline{O}_{t+1}]_{i,:} = \sqrt{\tfrac{n}{1-\beta_2^{t+1}}}e_1^{(n)}$ for all $i=1,\ldots, m$. In addition, by the definition of $D_\star$ (see Eq.~\eqref{eq:def-H-D}), we have
\begin{equation*}
\overline{O}_{t+1}=\sqrt{\tfrac{mn}{1-\beta_2^{t+1}}}D_\star, 
\end{equation*}
which further implies 
\begin{equation*}
\|\overline{O}_{t+1}\|_F=\sqrt{\tfrac{mn}{1-\beta_2^{t+1}}},\quad D_{t+1}=0.2\sqrt{mn}\tfrac{\overline{O}_{t+1}}{\|\overline{O}_{t+1}\|_F}=0.2\sqrt{mn}D_\star.
\end{equation*}
By induction, we have $X_{t+1}=X_t-\eta_{t}D_{t+1}=X_0+q_tD_\star-\eta_{t}D_{t+1}=X_0+q_{t+1}D_\star$. 

We prove (ii). Since $f'$ is a maximum of $L$-Lipschitz functions, we obtain that $f'$ is $L$-Lipschitz. Since $\nabla F(X)$ and $\nabla F(Y)$ are supported in the first column for all $X,Y \in \br^{m\times n}$, we have
\begin{eqnarray*}
\lefteqn{\|\nabla F(X)-\nabla F(Y)\|_\textnormal{nuc}^2 = \|\nabla F(X)-\nabla F(Y)\|_F^2} \\
& = & |f'(\langle D_\star,X-X_0\rangle)-f'(\langle D_\star,Y-X_0\rangle)|^2 + \tfrac{L^2|\langle H_\star-\lambda D_\star,X-Y\rangle|^2}{\|H_\star-\lambda D_\star\|_F^2} \\
& \leq & L^2\left(|\langle D_\star,X-Y\rangle|^2 + \tfrac{|\langle H_\star-\lambda D_\star,X-Y\rangle|^2}{\|H_\star-\lambda D_\star\|_F^2}\right) \\
& \leq & L^2\|(X-Y)e_1^{(n)}\|_2^2 \leq L^2\|X-Y\|_\textnormal{op}^2,
\end{eqnarray*}
where we used the fact that $D_\star$ and $\frac{H_\star-\lambda D_\star}{\|H_\star-\lambda D_\star\|_F}$ are orthonormal and supported in the first column.

We prove (iii). Since $f'\geq 0$, it follows from Eq.~\eqref{eq:lower-interpolant} that
\begin{eqnarray*}
\lefteqn{f(0)-\inf_{q \in \br} f(q) = \int_{-\infty}^0\max_{0\leq j\leq T_0}\max\{2\epsilon\lambda-L|z-q_j|,0\}dz} \\
& \leq & \int_{-\infty}^0 \max\{2\epsilon\lambda-L|q|,0\}dq+\sum_{j=1}^{T_0}\int_{-\infty}^{+\infty}\max\{2\epsilon\lambda-L|q-q_j|,0\}dq \\
& = & \tfrac{2\epsilon^2\lambda^2}{L}+\tfrac{4T_0\epsilon^2\lambda^2}{L}.
\end{eqnarray*}
This implies that $f$ is bounded below. By definition, $F$ is bounded below and $F^\star$ is finite. 

Since $F(X_0)=f(0)=0$, Eq.~\eqref{eq:lower-objective} gives
\begin{equation*}
F(X_0)-F^\star\leq\tfrac{2\epsilon^2\lambda^2}{L}+\tfrac{4T_0\epsilon^2\lambda^2}{L}+\tfrac{2\epsilon^2\|H_\star-\lambda D_\star\|_F^2}{L}=\tfrac{2\epsilon^2}{L}+\tfrac{4T_0\epsilon^2\lambda^2}{L}\leq\tfrac{2\epsilon^2}{L}+\tfrac{16T_0\epsilon^2}{mL}<\tfrac{5\Delta}{8}<\Delta,
\end{equation*}
where we used $\lambda^2<\frac{4}{m}$, $T_0\leq\frac{m\Delta L}{32\epsilon^2}$, and $\epsilon<\frac{1}{4}\sqrt{\Delta L}$.
 
\subsection{Proof of Theorems~\ref{theorem:normuon-upper-deterministic} and~\ref{theorem:normuon-upper-stochastic}}
We first prove the descent inequality in Eq.~\eqref{eq:main-descent-ineq} for the sake of completeness.
\begin{lemma}\label{lem:smooth-descent}
If Assumption~\ref{assumption:smooth} holds with $L>0$, we have
\begin{equation}\label{eq:upper-smooth-descent}
F(Y)\leq F(X)+\langle\nabla F(X),Y-X\rangle+\tfrac{L}{2}\|Y-X\|_\textnormal{op}^2 \textnormal{ for any } X,Y \in \br^{m\times n}. 
\end{equation}
In addition, $F(X)-F^\star \leq \Delta$ implies $\|\nabla F(X)\|_\textnormal{nuc} \leq \sqrt{2\Delta L}$.
\end{lemma}
\begin{proof}
For any $X,Y \in \br^{m\times n}$, we have 
\begin{equation*}
F(Y)-F(X)=\int_0^1 \langle\nabla F((1-s)X+sY),Y-X\rangle ds \leq \langle \nabla F(X), Y-X\rangle + \tfrac{L}{2}\|Y-X\|_\textnormal{op}^2. 
\end{equation*}
Suppose that $\nabla F(X)=U\Sigma V^\top$ is a compact SVD. Then, we let $Y=X-\frac{1}{L}\|\nabla F(X)\|_\textnormal{nuc}UV^\top$ and obtain that $F^\star\leq F(Y)\leq F(X)-\frac{\|\nabla F(X)\|_\textnormal{nuc}^2}{2L}$.  
\end{proof}
We then bound $\langle M_{t+1},D_{t+1}\rangle$ in the following lemma. 
\begin{lemma}
\label{lem:direction}
If Assumption~\ref{assumption:polar} holds and $\ell \in [0,\frac{1}{d})$, we have $\langle M_{t+1},D_{t+1}\rangle \geq 0.2\kappa\sqrt{mn}\|M_{t+1}\|_\textnormal{nuc}$ and $\|D_{t+1}\|_\textnormal{op}\leq\|D_{t+1}\|_F=0.2\sqrt{mn}$ for any $\alpha \geq 0$ and any $t \geq 0$. 
\end{lemma}
\begin{proof}
If $M_{t+1}=0$, we have $D_{t+1}=\overline{O}_{t+1}=O_{t+1}=0$ which yields the desired result. Otherwise, by the update rule of Algorithm~\ref{algorithm:normuon}, we have
\begin{equation}
\label{eq:lemma-direction-main}
\langle M_{t+1},D_{t+1}\rangle = \tfrac{0.2\sqrt{mn}\langle M_{t+1},\overline{O}_{t+1}\rangle}{\|\overline{O}_{t+1}\|_F}.
\end{equation}
First, we upper bound $\|\overline{O}_{t+1}\|_F$. By Assumption~\ref{assumption:polar}, we have $O_{t+1}=U\widetilde{\Sigma}V^\top$ where $M_{t+1}=U\Sigma V^\top$ is the compact SVD. For all $t$, the largest singular value of $\frac{M_{t+1}}{\|M_{t+1}\|_F}$ is at least $\frac{1}{\sqrt{d}}>\ell$, which implies that $1-\delta\leq\|O_{t+1}\|_\textnormal{op} \leq 1+\delta$. Since $[v_{t+1}]_i= \beta_2[v_t]_i+\frac{1-\beta_2}{n}\|[O_{t+1}]_{i,:}\|_2^2$ and $v_0=0$, we have
\begin{equation*}
\tfrac{1-\beta_2}{n}\|[O_{t+1}]_{i,:}\|_2^2\leq [v_{t+1}]_i \leq \tfrac{(1+\delta)^2}{n}, \quad \textnormal{for all } i. 
\end{equation*}
By the definition of $\overline{O}_{t+1}$, we have
\begin{equation*}
\|[\overline{O}_{t+1}]_{i,:}\|_2 = \tfrac{\|[O_{t+1}]_{i,:}\|_2}{\sqrt{[v_{t+1}]_i}+\alpha} \leq \tfrac{\|[O_{t+1}]_{i,:}\|_2}{\sqrt{\frac{1-\beta_2}{n}}\|[O_{t+1}]_{i,:}\|_2 + \alpha } \leq \tfrac{1+\delta}{\sqrt{\frac{1-\beta_2}{n}}(1+\delta)+\alpha}, \quad \textnormal{for all } i, 
\end{equation*}
where the last inequality follows from monotonicity and $\|[O_{t+1}]_{i,:}\|_2\leq \|O_{t+1}\|_\textnormal{op}\leq1+\delta$. Thus, we have 
\begin{equation}
\label{eq:barOF}
\|\overline{O}_{t+1}\|_F \leq \tfrac{(1+\delta)\sqrt{m}}{\sqrt{\frac{1-\beta_2}{n}}(1+\delta)+\alpha}.
\end{equation}
Second, we lower bound $\langle M_{t+1},\overline{O}_{t+1}\rangle$. Since $ M_{t+1}O_{t+1}^\top = U\Sigma\widetilde{\Sigma}U^\top\succeq 0$, for $i=1,\ldots,m$, we have $\langle [M_{t+1}]_{i,:}, [O_{t+1}]_{i,:}\rangle\geq 0$. Then, we have
\begin{small}
\begin{equation} \label{eq:MTbarO}
\begin{array}{rcl}
\langle M_{t+1},\overline{O}_{t+1}\rangle & = & \sum_{i=1}^m \langle [M_{t+1}]_{i,:}, [\overline{O}_{t+1}]_{i,:}\rangle = \sum_{i=1}^m \langle [M_{t+1}]_{i,:}, \tfrac{1}{\sqrt{[v_{t+1}]_i}+\alpha}[O_{t+1}]_{i,:}\rangle \\
& \geq & \sum_{i=1}^m  \tfrac{1}{\frac{1+\delta} {\sqrt{n}}+\alpha} \langle [M_{t+1}]_{i,:}, [O_{t+1}]_{i,:}\rangle = \tfrac{1}{\frac{1+\delta} {\sqrt{n}}+\alpha}\langle M_{t+1},O_{t+1}\rangle.
\end{array}
\end{equation}
\end{small}
By Assumption~\ref{assumption:polar}, we obtain from $\Sigma_{jj}\geq \ell \|M_{t+1}\|_F$ that $\widetilde{\Sigma}_{jj}\geq 1-\delta$. It follows that
\begin{equation}\label{eq:MTO}
\begin{array}{rcl}
\langle M_{t+1},O_{t+1}\rangle & \geq & (1-\delta)\sum_{j:\,\Sigma_{jj} \geq \ell\|M_{t+1}\|_F}\Sigma_{jj} \geq (1-\delta)(\|M_{t+1}\|_\textnormal{nuc}-d\ell\|M_{t+1}\|_F) \\
& \geq & (1-\delta)(1-d\ell)\|M_{t+1}\|_\textnormal{nuc},
\end{array}
\end{equation}
where the last inequality uses $\|M_{t+1}\|_F \leq \|M_{t+1}\|_\textnormal{nuc}$. Combining Eq.~\eqref{eq:lemma-direction-main}, Eq.~\eqref{eq:barOF}, Eq.~\eqref{eq:MTbarO}, and Eq.~\eqref{eq:MTO} yields
\begin{equation*}
\tfrac{\langle M_{t+1},D_{t+1}\rangle}{0.2\sqrt{mn}} \geq \tfrac{\sqrt{\frac{1-\beta_2}{n}}(1+\delta)+\alpha}{(1+\delta)\sqrt{m}\,(\frac{1+\delta}{\sqrt{n}}+\alpha)}(1-\delta)(1-d\ell)\|M_{t+1}\|_\textnormal{nuc} \geq \kappa\|M_{t+1}\|_\textnormal{nuc}.
\end{equation*}
By definition, we have $\|D_{t+1}\|_\textnormal{op}\leq\|D_{t+1}\|_F=0.2\sqrt{mn}$. 
\end{proof}

\noindent\emph{Proof of Theorem~\ref{theorem:normuon-upper-deterministic}.}
Since $\beta_1=0$ and $G_t=\nabla F(X_t)$, we have $M_{t+1}=\nabla F(X_t)$. By applying Lemmas~\ref{lem:smooth-descent} and~\ref{lem:direction} and using $X_{t+1}=X_t-\eta D_{t+1}$, we have
\begin{equation*}
\begin{array}{lcl}
F(X_{t+1})-F(X_t) & \leq & -\eta\langle\nabla F(X_t),D_{t+1}\rangle+\tfrac{L\eta^2}{2}\|D_{t+1}\|_\textnormal{op}^2 \\
& \leq & -0.2\eta\kappa\sqrt{mn}\|\nabla F(X_t)\|_\textnormal{nuc}+\tfrac{L\eta^2}{2}(0.2\sqrt{mn})^2\\
& = & -\tfrac{\kappa^2\epsilon}{L}\|\nabla F(X_t)\|_\textnormal{nuc}+\tfrac{\kappa^2\epsilon^2}{2L}.
\end{array}
\end{equation*}
Summing the above inequality over $t=0,1,\ldots,T-1$ and using $F(X_0)-F^\star \leq \Delta$ yields
\begin{equation*}
\min_{0\leq t<T}\|\nabla F(X_t)\|_\textnormal{nuc} \leq \tfrac{1}{T} \sum_{t=0}^{T-1}\|\nabla F(X_t)\|_\textnormal{nuc}\leq\tfrac{\Delta L}{\kappa^2\epsilon T}+\tfrac{\epsilon}{2}. 
\end{equation*} 
We choose $T:=\lceil\frac{2\Delta L}{\kappa^2\epsilon^2}\rceil$. By the definition of $\kappa$, we have $\min_{0\leq t<T}\|\nabla F(X_t)\|_\textnormal{nuc} \leq \epsilon$ and $T=O(\frac{m\Delta L}{\epsilon^2})$. Since Algorithm~\ref{algorithm:normuon} calls the gradient oracle once per iteration, the total number of calls to the gradient oracle is bounded by $O(\frac{m\Delta L}{\epsilon^2})$. \hfill$\square$

\noindent\emph{Proof of Theorem~\ref{theorem:normuon-upper-stochastic}.}
By applying Lemma~\ref{lem:smooth-descent} and using $X_{t+1}=X_t-\eta D_{t+1}$, we have
\begin{equation}\label{eq:decsent-lemma-app}
F(X_{t+1})-F(X_t)\leq-\eta\langle\nabla F(X_t),D_{t+1}\rangle+\tfrac{L\eta^2}{2}\|D_{t+1}\|_\textnormal{op}^2.
\end{equation}
We write $M_{t+1}=\nabla F(X_t)+E_t+B_t$, where
\begin{equation*}
\begin{array}{lcl}
E_t & = & -\beta_1^{t+1}\nabla F(X_0)+\sum_{s=1}^t\beta_1^{t-s+1}(\nabla F(X_{s-1})-\nabla F(X_s)), \\
B_t & = & (1-\beta_1)\sum_{s=0}^t\beta_1^{t-s}(G_s-\nabla F(X_s)).
\end{array}
\end{equation*}
By using $\|D_{t+1}\|_\textnormal{op}\leq\|D_{t+1}\|_F=0.2\sqrt{mn}$ (see Lemma~\ref{lem:direction}) and $\kappa \leq \frac{1}{\sqrt{d}}$, we have
\begin{eqnarray*}
\lefteqn{\tfrac{\langle\nabla F(X_t),D_{t+1}\rangle}{0.2\sqrt{mn}} \geq \kappa\|M_{t+1}\|_\textnormal{nuc}-\|E_t\|_\textnormal{nuc}-\|B_t\|_F} \\
& \geq & \kappa\|\nabla F(X_t)\|_\textnormal{nuc}-(1+\kappa)\|E_t\|_\textnormal{nuc}-\|B_t\|_F-\kappa\|B_t\|_\textnormal{nuc} \\
& \geq & \kappa\|\nabla F(X_t)\|_\textnormal{nuc}-2\|E_t\|_\textnormal{nuc}-2\|B_t\|_F.
\end{eqnarray*}
Combining this inequality with Eq.~\eqref{eq:decsent-lemma-app} and using $\|D_{t+1}\|_\textnormal{op} \leq 0.2\sqrt{mn}$ (see Lemma~\ref{lem:direction}) and the definition of $\eta$ yields
\begin{equation*}
F(X_{t+1})-F(X_t) \leq -\tfrac{\kappa(1-\beta_1)\epsilon}{16L}\big(\kappa\|\nabla F(X_t)\|_\textnormal{nuc}-2\|E_t\|_\textnormal{nuc}-2\|B_t\|_F\big)+\tfrac{\kappa^2(1-\beta_1)^2\epsilon^2}{512L}. 
\end{equation*}
Summing the above inequality over $t=0,\ldots,T-1$ and using $F(X_0)-F^\star \leq \Delta$ yields
\begin{equation}\label{eq:upper-summed-descent}
\kappa\textstyle\sum_{t=0}^{T-1}\|\nabla F(X_t)\|_\textnormal{nuc}\leq\tfrac{16\Delta L}{\kappa(1-\beta_1)\epsilon}+2\textstyle\sum_{t=0}^{T-1}\|E_t\|_\textnormal{nuc}+2\textstyle\sum_{t=0}^{T-1}\|B_t\|_F+\tfrac{\kappa(1-\beta_1)\epsilon T}{32}.
\end{equation}
First, we bound $\sum_{t=0}^{T-1}\|E_t\|_\textnormal{nuc}$. Indeed, we have $\|\nabla F(X_0)\|_\textnormal{nuc} \leq \sqrt{2\Delta L}$ (see Lemma~\ref{lem:smooth-descent}) and obtain from Assumption~\ref{assumption:smooth} and $\|X_s-X_{s-1}\|_\textnormal{op} \leq 0.2\eta\sqrt{mn}$ that $\|\nabla F(X_s)-\nabla F(X_{s-1})\|_\textnormal{nuc} \leq \frac{\kappa(1-\beta_1)\epsilon}{16}$ for $1\leq s\leq T$. This implies 
\begin{equation*}
\|E_t\|_\textnormal{nuc} \leq \beta_1^{t+1}\sqrt{2\Delta L}+\tfrac{\kappa(1-\beta_1)\epsilon}{16}\textstyle\sum_{s=1}^t\beta_1^{t-s+1} \leq \beta_1^{t+1}\sqrt{2\Delta L}+\tfrac{\kappa\epsilon}{16},
\end{equation*}
and
\begin{equation*}
\tfrac{1}{T}\sum_{t=0}^{T-1}\|E_t\|_\textnormal{nuc} \leq \tfrac{\sqrt{2\Delta L}}{(1-\beta_1)T}+\tfrac{\kappa\epsilon}{16}.
\end{equation*}
Second, we bound $\EE[\|B_t\|_F]$. Indeed, we let $\FCal_s$ be the sigma algebra generated by $X_0,\xi_0,\ldots,\xi_{s-1}$. The independence of $\xi_s$ and $\FCal_s$ and Assumption~\ref{assumption:noise} give $\EE[G_s-\nabla F(X_s)|\FCal_s]=0$ and $\EE[\|G_s-\nabla F(X_s)\|_\textnormal{nuc}^2 | \FCal_s] \leq \sigma^2$. Since $G_s-\nabla F(X_s)$ is $\FCal_j$-measurable and $\EE[\langle G_s-\nabla F(X_s),G_j-\nabla F(X_j)\rangle]=0$ for $s<j$, we have
\begin{equation*}
\begin{array}{rcl}
\EE[\|B_t\|_F^2] & = & (1-\beta_1)^2\sum_{s=0}^t\beta_1^{2(t-s)}\EE[\|G_s-\nabla F(X_s)\|_F^2] \ \leq \ (1-\beta_1)^2\sigma^2\sum_{s=0}^t\beta_1^{2(t-s)} \\ 
& \leq & (1-\beta_1)\sigma^2.
\end{array}
\end{equation*}
Taking expectations in Eq.~\eqref{eq:upper-summed-descent} and putting these pieces together, we have
\begin{equation}\label{eq:upper-weighted-master}
\tfrac{1}{T}\sum_{t=0}^{T-1} \EE[\|\nabla F(X_t)\|_\textnormal{nuc}] \leq \tfrac{16\Delta L}{\kappa^2(1-\beta_1)\epsilon T}+\tfrac{2\sqrt{2\Delta L}}{\kappa(1-\beta_1)T}+\tfrac{5\epsilon}{32}+\tfrac{2\sigma\sqrt{1-\beta_1}}{\kappa}.
\end{equation}
We choose $T:=\lceil\frac{64\Delta L}{\kappa^2(1-\beta_1)\epsilon^2}\rceil$. Using the definition of $\beta_1$ and the condition $\epsilon\leq\sqrt{\Delta L}$, we have
\begin{equation*}
\EE[\|\nabla F(\widetilde{X})\|_\textnormal{nuc}] = \tfrac{1}{T}\sum_{t=0}^{T-1}\EE[\|\nabla F(X_t)\|_\textnormal{nuc}] \leq \left(\tfrac{1}{4}+\tfrac{\sqrt{2}}{32}+\tfrac{5}{32}+\tfrac{1}{4}\right)\epsilon < \epsilon.
\end{equation*}
In addition, we have $\frac{1}{1-\beta_1}=\max\{1,\frac{64\sigma^2}{\kappa^2\epsilon^2}\} \leq 1+\frac{64\sigma^2}{\kappa^2\epsilon^2}$. By the definition of $\kappa$, we have
\begin{equation*}
T=\left\lceil\tfrac{64\Delta L}{\kappa^2(1-\beta_1)\epsilon^2}\right\rceil\leq1+\tfrac{64\Delta L}{\kappa^2\epsilon^2}+\tfrac{4096\Delta L\sigma^2}{\kappa^4\epsilon^4}=O\left(\tfrac{m\Delta L}{\epsilon^2}+\tfrac{m^2\Delta L\sigma^2}{\epsilon^4}\right).
\end{equation*}
Since Algorithm~\ref{algorithm:normuon} calls the stochastic gradient oracle once per iteration, the total number of calls to the stochastic gradient oracle is bounded by $O(\frac{m\Delta L}{\epsilon^2}+\frac{m^2\Delta L\sigma^2}{\epsilon^4})$.
\hfill$\square$

\section{Additional Experiments} \label{app:exp}
We present detailed setups and additional results for our experiments.
\subsection{Setup and additional results for the synthetic experiment}
\paragraph{Setup.} 
We run Muon, NorMuon (Algorithm~\ref{algorithm:normuon}), and practical NorMuon with the scaling in Eq.~\eqref{eq:normuon-practical} on the same 150 fixed random seeds. We use 50 of these seeds to tune the constant step size and momentum parameters of each method for each $m$. For both Muon and NorMuon, we search over $\beta_1\in\{0,0.5,0.9,0.95\}$, and for NorMuon we additionally search over $\beta_2\in\{0,0.9,0.95,0.99\}$. We search over the effective step sizes $\gamma_j=10^{-5+0.5j}$ for $j=0,\ldots,10$. To match the global update scale across methods, we set $\eta=\gamma_j$ for Muon and NorMuon with Eq.~\eqref{eq:normuon-practical}, and $\eta=\frac{\gamma_j}{0.2\sqrt{m}}$ for NorMuon (Algorithm~\ref{algorithm:normuon}). For each method and $m$, we select the hyperparameters that achieve the smallest minimum loss $\min_{0\leq t\leq T} F(X_t)$ over $T=1000$ iterations. We then evaluate all methods with their selected hyperparameters on the remaining 100 seeds. The selected hyperparameters are reported in Table~\ref{tab:synthetic-hyperparameters}, and the loss curves are shown in Figure~\ref{fig:synthetic}. As $m$ grows, NorMuon selects larger step sizes. Hence, its loss decreases faster initially but stalls later.
\begin{table}[!t]
\centering
\caption{Selected hyperparameters $(\gamma,\beta_1,\beta_2)$ for the synthetic experiment.}
\label{tab:synthetic-hyperparameters}
\begin{tabular}{lccc}
\toprule
$m$ 
& Muon 
& \shortstack{NorMuon\\(Algorithm~\ref{algorithm:normuon})}
& \shortstack{NorMuon with\\Eq.~\eqref{eq:normuon-practical}} \\
\midrule
$8$    
& $(10^{-3}, 0, \text{--})$
& $(10^{-3}, 0, 0.95)$
& $(10^{-3}, 0, 0)$ \\
$32$   
& $(10^{-3}, 0, \text{--})$
& $(10^{-2.5}, 0.5, 0.99)$
& $(10^{-2.5}, 0.5, 0.99)$ \\
$128$  
& $(10^{-3}, 0, \text{--})$
& $(10^{-2.5}, 0.5, 0.9)$
& $(10^{-2.5}, 0.5, 0.9)$ \\
$512$  
& $(10^{-3}, 0, \text{--})$
& $(10^{-2}, 0.5, 0.9)$
& $(10^{-2}, 0.5, 0.9)$ \\
$2048$ 
& $(10^{-3}, 0, \text{--})$
& $(10^{-1.5}, 0.5, 0.9)$
& $(10^{-1.5}, 0.5, 0.9)$ \\
\bottomrule
\end{tabular}
\end{table}
\begin{figure}[!t]
\centering
\includegraphics[width=\linewidth]{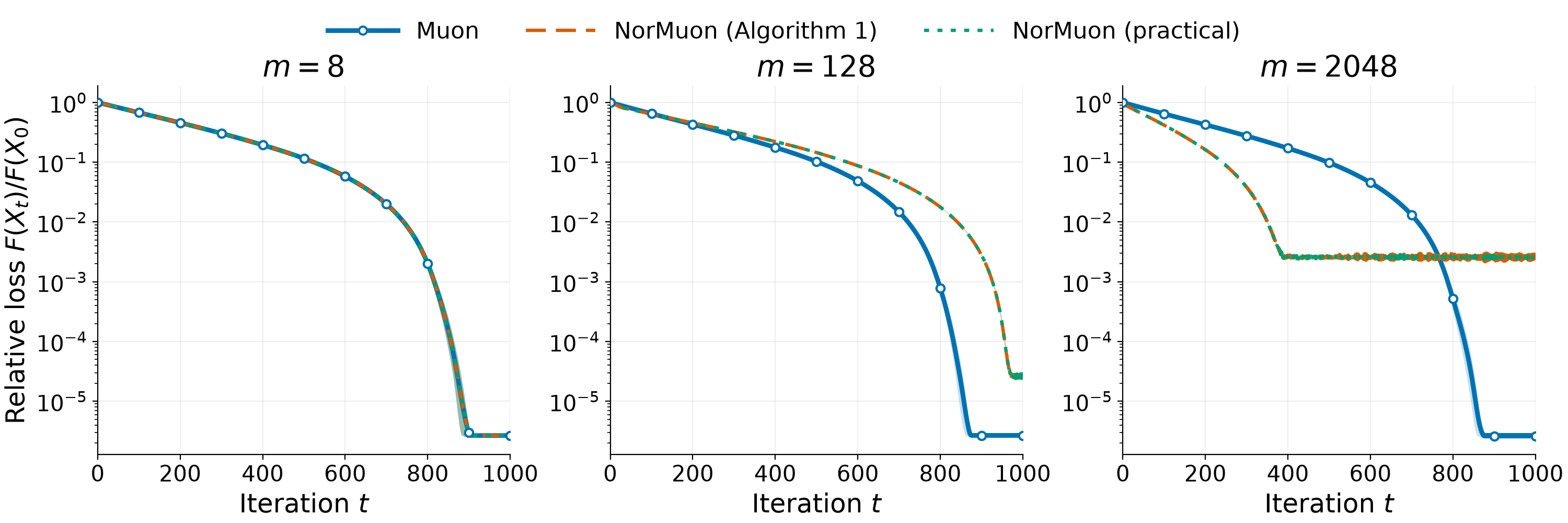}
\caption{Loss curves for the synthetic experiment.}
\label{fig:synthetic}
\end{figure}
\paragraph{Cosine schedule.}
We additionally compare Muon and NorMuon under a cosine step-size schedule on the same hard instance and observe a similar dimension-dependent slowdown for NorMuon. The setup remains the same except that $T=25$ and we search over effective peak step sizes $\gamma_j=10^{-5+0.5j}$ for $j=0,\ldots,12$ for all methods. We use the cosine schedule $\eta_t=\frac{\eta}{2}\left(1+\cos\frac{\pi t}{T}\right)$ for $t=0,\ldots,T-1$, where the peak step size is $\eta=\gamma_j$ for Muon and NorMuon with Eq.~\eqref{eq:normuon-practical}, and $\eta=\gamma_j/(0.2\sqrt{m})$ for NorMuon in Algorithm~\ref{algorithm:normuon}. For each method and $m$, we select the hyperparameters that maximize the mean loss reduction $\log_{10}F(X_0)-\log_{10}\min_{0\leq t\leq T}F(X_t)$ over 50 tuning seeds. The selected hyperparameters are reported in Table~\ref{tab:synthetic-cosine-hyperparameters}. We then evaluate all methods with their selected hyperparameters on the remaining 100 seeds. The evaluation results are reported in Table~\ref{tab:synthetic-cosine-iterations}, and the loss curves are shown in Figure~\ref{fig:synthetic-cosine}. 
\begin{table}[!t]
\centering
\caption{Selected hyperparameters $(\gamma,\beta_1,\beta_2)$
for the synthetic experiment with a cosine step-size schedule.}
\label{tab:synthetic-cosine-hyperparameters}
\begin{tabular}{lccc}
\toprule
$m$ & Muon
& \shortstack{NorMuon\\(Algorithm~\ref{algorithm:normuon})}
& \shortstack{NorMuon with\\Eq.~\eqref{eq:normuon-practical}} \\
\midrule
$8$
& $(10^{-1},0,\text{--})$
& $(10^{-1},0,0.95)$
& $(10^{-1},0,0)$ \\
$32$
& $(10^{-1},0,\text{--})$
& $(10^{-0.5},0,0.9)$
& $(10^{-0.5},0,0.9)$ \\
$128$
& $(10^{-1},0,\text{--})$
& $(10^{-0.5},0,0.9)$
& $(10^{-0.5},0,0.9)$ \\
$512$
& $(10^{-1},0,\text{--})$
& $(10^{0},0,0.99)$
& $(10^{0},0,0.99)$ \\
$2048$
& $(10^{-1},0,\text{--})$
& $(10^{0},0,0)$
& $(10^{0},0,0)$ \\
\bottomrule
\end{tabular}
\end{table}

\begin{table}[!t]
\centering
\caption{Mean $\pm$ $2\times$ standard error of the loss reduction $\log_{10}F(X_0)-\log_{10}\min_{0\leq t\leq T}F(X_t)$ over 100 evaluation seeds under a cosine schedule with $T=25$. A larger value means better performance.}
\label{tab:synthetic-cosine-iterations}
\begin{tabular}{lrrrrr}
\toprule
$m$ & $8$ & $32$ & $128$ & $512$ & $2048$ \\
\midrule
Muon
& $5.51_{\pm 0.05}$ & $5.57_{\pm 0.04}$ & $5.55_{\pm 0.05}$
& $5.57_{\pm 0.04}$ & $5.57_{\pm 0.05}$ \\
NorMuon (Algorithm~\ref{algorithm:normuon})
& $5.51_{\pm 0.05}$ & $4.64_{\pm 0.04}$ & $4.75_{\pm 0.05}$
& $3.63_{\pm 0.05}$ & $3.58_{\pm 0.05}$ \\
NorMuon with Eq.~\eqref{eq:normuon-practical}
& $5.51_{\pm 0.05}$ & $4.64_{\pm 0.04}$ & $4.75_{\pm 0.05}$
& $3.63_{\pm 0.05}$ & $3.58_{\pm 0.05}$ \\
\bottomrule
\end{tabular}
\end{table}

\begin{figure}[!t]
\centering
\includegraphics[width=\linewidth]{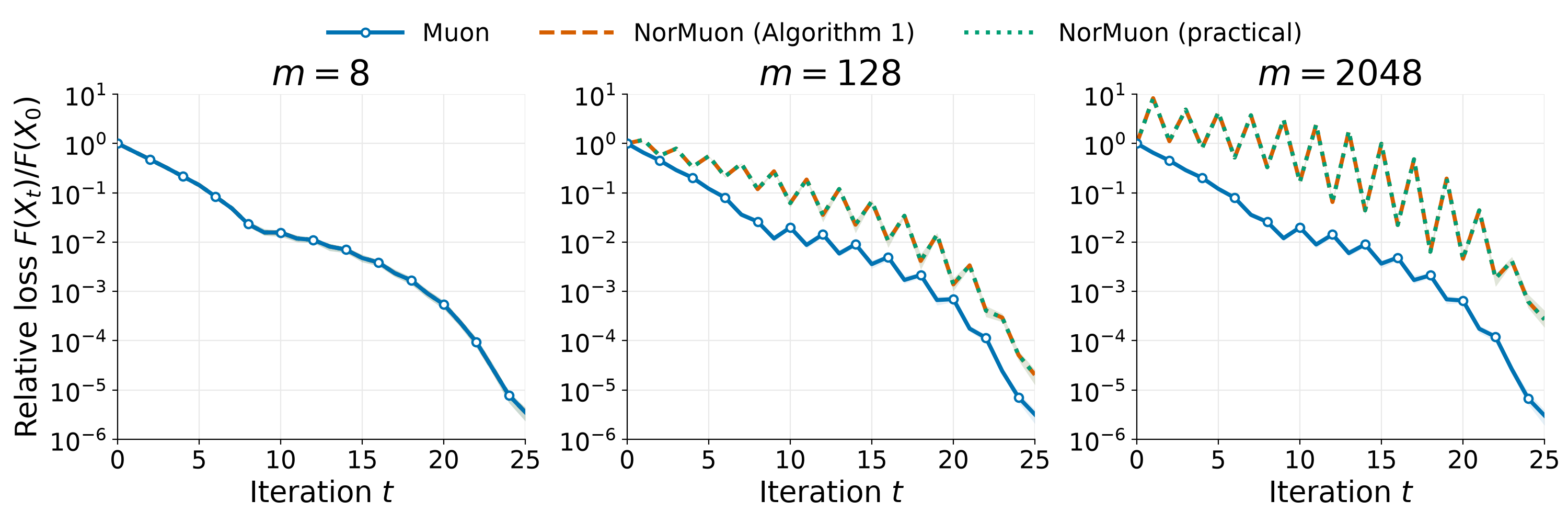}
\caption{Loss curves for the cosine schedule synthetic experiment with $T=25$. }
\label{fig:synthetic-cosine}
\end{figure}

\subsection{Setup for image classification}
For Muon, NorMuon (with both column and row normalization), and AdamW, we use a step-size schedule with $5\%$ linear warmup followed by cosine decay to $0$. Muon and NorMuon use 5 PolarExpress iterations~\citep{Amsel-2026-Polar}. For AdamW, we sweep the step size over $\{0.0003, 0.001, 0.003, 0.01\}$ with momentum factors $(\beta_1, \beta_2)=(0.9,0.999)$. For Muon and NorMuon, we sweep the step size over $\{0.03, 0.05, 0.10, 0.15\}$ and the momentum factor over $\{0.5, 0.6, 0.7, 0.8, 0.9, 0.95\}$. NorMuon additionally uses $\beta_2=0.9$. For all methods, we sweep the decoupled weight decay over $\{0, 0.01\}$. We use label smoothing with a factor of 0.2 and clip the Euclidean norm of the gradient to at most $1$.

\subsection{Setup and additional results for LLM pretraining} 
We pretrain decoder-only transformers with the nanochat codebase~\citep{Karpathy-2025-Nanochat} at depths $12$, $24$, and $32$. All models use $\operatorname{ReLU}^2$ activations, root mean square (RMS) normalization \citep{Zhang-2019-Root}, rotary positional embeddings~\citep{Su-2024-Roformer}, query-key normalization~\citep{Henry-2020-Query}, untied token embeddings and language-model (LM) heads, and separate value embeddings~\citep{Zhou-2025-Value}. Following the standard design of Muon~\citep{Jordan-2024-Muon}, Muon or NorMuon updates only the transformer's weight matrices, while AdamW updates the embeddings, LM head, and other trainable parameters. Table~\ref{tab:llm_configs} summarizes the model configurations and training token budgets for the experiments in Table~\ref{tab:llm_results} and Figure~\ref{fig:llm_results}. We follow the step-size schedule of nanochat~\citep{Karpathy-2025-Nanochat}: the step size is warmed up linearly over the first $40$ updates, held constant until $35\%$ of training, and then decayed linearly to $5\%$ of its peak value over the remaining $65\%$ of training. The original NorMuon~\citep{Li-2026-Normuon} rescales each row as $[\overline{O}_{t+1}]_{i,:}=[O_{t+1}]_{i,:}/(\sqrt{v_{t+1,i}}+\alpha)$ with $\alpha = 10^{-10}$. See Algorithm~\ref{algorithm:normuon} for details. In our experiments, we instead use $[\overline{O}_{t+1}]_{i,:} = [O_{t+1}]_{i,:}/\sqrt{\max\{v_{t+1,i},\alpha\}}$ with the same $\alpha$, following~\cite{Karpathy-2025-Nanochat}.

\begin{table}[!t] 
\centering 
\caption{Model configurations and training budgets for LLM pretraining.}
\label{tab:llm_configs}
\begin{tabular}{lcccccccc}
\toprule
Model & Depth & \shortstack{Hidden\\dim.} & \shortstack{MLP\\dim.} & \shortstack{Vocab.\\size} & \shortstack{Seq.\\length} & \shortstack{Matrix\\params} & \shortstack{Batch size\\(tokens)} & \shortstack{Training\\tokens} \\
\midrule
$286$M  & $12$ & $768$  & $3072$ & $32{,}768$  & $2048$ & $84.9$M & $524{,}288$     & $2.20$B \\
$1.38$B & $24$ & $1536$ & $6144$ & $32{,}768$  & $2048$ & $679$M  & $1{,}048{,}576$ & $14.6$B \\
$6.44$B & $32$ & $2048$ & $8192$ & $131{,}072$ & $4096$ & $1.61$B & $1{,}048{,}576$ & $19.7$B \\
\bottomrule
\end{tabular}
\end{table}

\paragraph{Random seeds.} For the 1.38B model, we run Muon and NorMuon with 3 random seeds and report the mean $\pm$ standard deviation in Table~\ref{tab:appendix-muon-normuon}. The average improvement of NorMuon over Muon exceeds one standard deviation. Due to space constraints, Table~\ref{tab:llm_results} reports only the mean.
\begin{table}[t!]
\caption{Mean $\pm$ standard deviation of LLM pretraining downstream accuracies and test losses.}
\label{tab:appendix-muon-normuon}
\centering
\begin{tabular}{@{}lcc@{}}
\toprule
Benchmark & Muon & NorMuon \\
\midrule
MMLU       & \textbf{33.3}$_{\pm 0.3}$ & \textbf{33.3}$_{\pm 0.06}$ \\
HellaSwag  & 60.3$_{\pm 0.2}$          & \textbf{60.9}$_{\pm 0.4}$ \\
PIQA       & 75.9$_{\pm 0.3}$          & \textbf{76.1}$_{\pm 0.4}$ \\
WinoGrande & 55.5$_{\pm 0.4}$          & \textbf{58.4}$_{\pm 1.5}$ \\
ARC-C      & 43.8$_{\pm 0.8}$          & \textbf{45.4}$_{\pm 0.3}$ \\
ARC-E      & 73.6$_{\pm 0.3}$          & \textbf{75.0}$_{\pm 0.8}$ \\
BoolQ      & \textbf{64.4}$_{\pm 2.6}$ & 62.0$_{\pm 1.7}$ \\
CSQA       & \textbf{60.0}$_{\pm 1.0}$ & 59.6$_{\pm 0.5}$ \\
SIQA       & 49.1$_{\pm 0.1}$          & \textbf{50.2}$_{\pm 0.03}$ \\
OBQA       & 47.1$_{\pm 2.8}$          & \textbf{49.8}$_{\pm 2.1}$ \\
\midrule
Avg        & 56.3$_{\pm 0.4}$          & \textbf{57.1}$_{\pm 0.2}$ \\
Test loss  & 2.2914$_{\pm 0.0006}$     & \textbf{2.2810}$_{\pm 0.0009}$ \\
\bottomrule
\end{tabular}
\end{table}

\paragraph{Setup for ablation studies.} We conduct ablation studies on the $1.38$B model trained on the compute-optimal budget of $7.66$B tokens. Unless otherwise stated, we use a base step size of $0.02$, $5$ PolarExpress iterations, a weight decay of $0.02$, and Nesterov-type momentum with $\beta_1=0.95$.  

We next describe the setups for the other optimizers considered in Section~\ref{sec:exp}. For Aurora~\citep{Dewulf-2026-Aurora}, we use the $K=2$ variant and apply diagonal refinement to tall matrices. For Muon+, we use the column--row variant, which normalizes columns and then rows after the polar approximation. MuonEq-R instead equilibrates rows before the polar approximation. For the $286$M model, we sweep the base step size over $\{0.02,0.03,0.04\}$ for all methods, and all of them select $0.03$. For the $1.38$B model, we use a common base step size of $0.02$ for the reasons given below.

\paragraph{Step-size sensitivity under the compute-optimal setting.} We compare base step sizes in $\{0.01,\allowbreak 0.02,\allowbreak 0.03,0.04\}$ for the $1.38$B model in Figure~\ref{fig:appendix-ablations}(a) and Table~\ref{tab:appendix-step-size}. NorMuon attains lower validation and test losses than Muon at every step size tested. For NorMuon, step sizes $0.01$ and $0.02$ give nearly identical losses, and the loss increases for larger step sizes. Muon is nearly insensitive to step sizes between $0.01$ and $0.03$, attains its lowest loss at $0.03$, and degrades at $0.04$. Even when each optimizer uses its own best step size, NorMuon still attains a lower loss. We therefore use $0.02$ as the common step size for the ablation studies, since it lies in a good regime for both optimizers and enables a comparison at a matched step size.
\begin{figure}[!t] 
\centering 
\includegraphics[width=\linewidth]{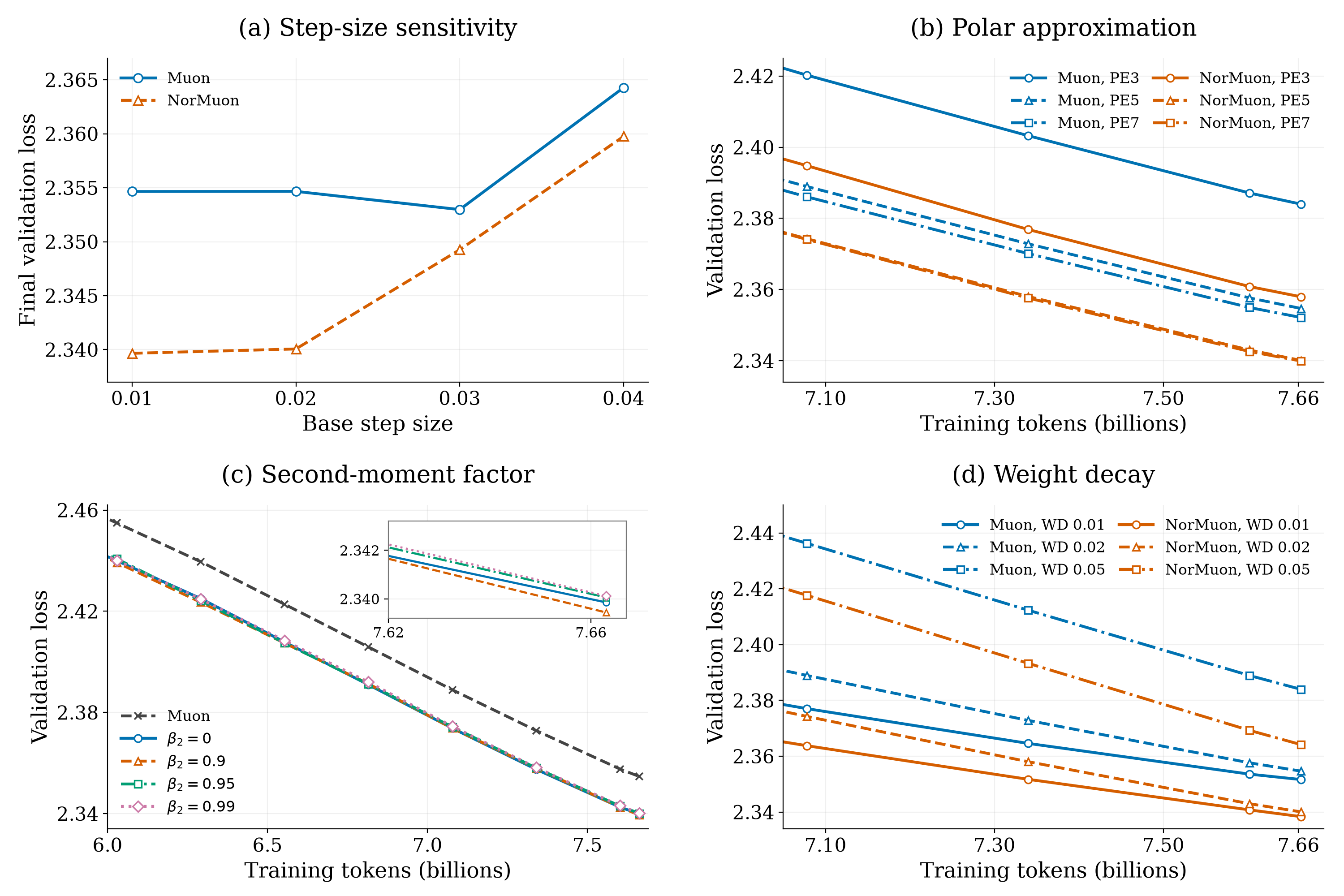}
\caption{Ablation studies on the $1.38$B model trained on $7.66$B tokens. (a) Final validation loss versus step size. (b) Muon and NorMuon with 3, 5, and 7 PolarExpress (PE) iterations. (c) NorMuon with different second-moment factors $\beta_2$. (d) Muon and NorMuon with different weight decay (WD) factors.} \label{fig:appendix-ablations}
\end{figure}
\begin{table}[!t]
\centering
\caption{Test loss versus base step size for the $1.38$B model trained on $7.66$B tokens.}
\label{tab:appendix-step-size}
\begin{tabular}{lcccc} \toprule
Step size & $0.01$ & $0.02$ & $0.03$ & $0.04$ \\ \midrule
Muon & $2.3587$ & $2.3587$ & $\mathbf{2.3571}$ & $2.3684$ \\
NorMuon & $\mathbf{2.3436}$ & $2.3442$ & $2.3533$ & $2.3637$ \\ \bottomrule
\end{tabular}
\end{table}

\paragraph{Accuracy of polar approximation.} We test whether the benefit of NorMuon persists across polar approximation accuracies by running Muon and NorMuon with $3$, $5$, or $7$ PolarExpress (PE) steps in Figure~\ref{fig:appendix-ablations}(b) and Table~\ref{tab:ablation-pe}. NorMuon achieves lower validation and test losses than Muon in every setting. Increasing the number of PE steps from $3$ to $5$ substantially reduces the test loss of both optimizers, while increasing it further to $7$ has little effect. The gap between Muon and NorMuon persists across all settings and is largest with $3$ PE steps. 
\begin{table}[!t]
\centering
\caption{Test loss versus the number of PolarExpress (PE) steps.}
\label{tab:ablation-pe}
\begin{tabular}{lccc}
\toprule
PE steps & $3$ & $5$ & $7$ \\
\midrule
Muon    & $2.3882$ & $2.3587$ & $2.3562$ \\
NorMuon & $2.3620$ & $2.3442$ & $2.3440$ \\
\bottomrule
\end{tabular}
\end{table}

\paragraph{Second-moment factor.} We compare second-moment factors $\beta_2\in\{0,0.9,0.95,0.99\}$ for NorMuon in Figure~\ref{fig:appendix-ablations}(c) and Table~\ref{tab:ablation-beta2}. NorMuon achieves nearly identical validation and test losses across all tested values of $\beta_2$ and outperforms Muon in every case.
\begin{table}[!t]
\centering
\caption{Test loss versus the second-moment factor $\beta_2$.}
\label{tab:ablation-beta2}
\begin{tabular}{lcccccc}
\toprule
NorMuon $\beta_2$ & $0$ & $0.9$ & $0.95$ & $0.99$ & Muon \\
\midrule
Test loss & $2.3440$ & $2.3435$ & $2.3442$ & $2.3439$ & $2.3587$ \\
\bottomrule
\end{tabular}
\end{table}

\paragraph{Weight decay.} We compare decoupled weight decay factors of $0.01$, $0.02$, and $0.05$ for both Muon and NorMuon in Figure~\ref{fig:appendix-ablations}(d) and Table~\ref{tab:ablation-wd}. NorMuon achieves lower test loss than Muon at every weight decay tested. 
\begin{table}[!t]
\centering
\caption{Test loss versus weight decay. }
\label{tab:ablation-wd}
\begin{tabular}{lccc}
\toprule
Weight decay & $0.01$ & $0.02$ & $0.05$ \\ \midrule
Muon & $2.3557$ & $2.3587$ & $2.3880$ \\
NorMuon & $2.3424$ & $2.3442$ & $2.3685$ \\
\bottomrule
\end{tabular}
\end{table}

\paragraph{Alignment.}
On the $286$M model, we compare Muon and the practical NorMuon variant in Eq.~\eqref{eq:normuon-practical} using the alignment $\langle M_{t+1},D_{t+1}\rangle$ from Lemma~\ref{lem:direction}. At a fixed time step $t+1$, let $M_\textnormal{M}$ and $M_\textnormal{N}$ denote the gradient momenta $M_{t+1}$ of Muon and NorMuon, respectively. The corresponding update directions are
\begin{equation*}
D_\textnormal{M}=\operatorname{Polar}_{\ell,\delta}\big(\tfrac{M_\textnormal{M}}{\|M_\textnormal{M}\|_F}\big),\qquad
D_\textnormal{N}=\tfrac{\|O\|_F}{\|\overline{O}\|_F}\overline{O}\ \ \textnormal{with}\ \ O=\operatorname{Polar}_{\ell,\delta}\big(\tfrac{M_\textnormal{N}}{\|M_\textnormal{N}\|_F}\big),
\end{equation*}
where $\overline{O}$ denotes the row-normalized version of $O$, computed as in Algorithm~\ref{algorithm:normuon}. We compare the relative difference in alignment, $\frac{\langle M_\textnormal{M},D_\textnormal{M}\rangle-\langle M_\textnormal{N},D_\textnormal{N}\rangle}{\langle M_\textnormal{M},D_\textnormal{M}\rangle}$, together with the analogous relative differences in $\|M\|_\textnormal{nuc}$, $\|D\|_\textnormal{op}$, $\|M\|_F$, and $\|D\|_F$ in Figure~\ref{fig:alignment}. A negative value indicates that the corresponding quantity is larger for NorMuon. On the MLP matrices, NorMuon tends to attain a slightly larger alignment $\langle M_\textnormal{N},D_\textnormal{N}\rangle$ after the first few hundred steps, along with a larger $\|M_\textnormal{N}\|_\textnormal{nuc}$, while on the QKVO matrices, the two optimizers have similar alignment. This is consistent with the observation of~\citet{Dewulf-2026-Aurora} that NorMuon mainly affects the MLP matrices. Finally, by the design of Eq.~\eqref{eq:normuon-practical}, the update directions of Muon and NorMuon have similar Frobenius norms, although NorMuon's updates have larger operator norms on the MLP matrices.
\begin{figure}[!t]
\centering
\includegraphics[width=\linewidth]{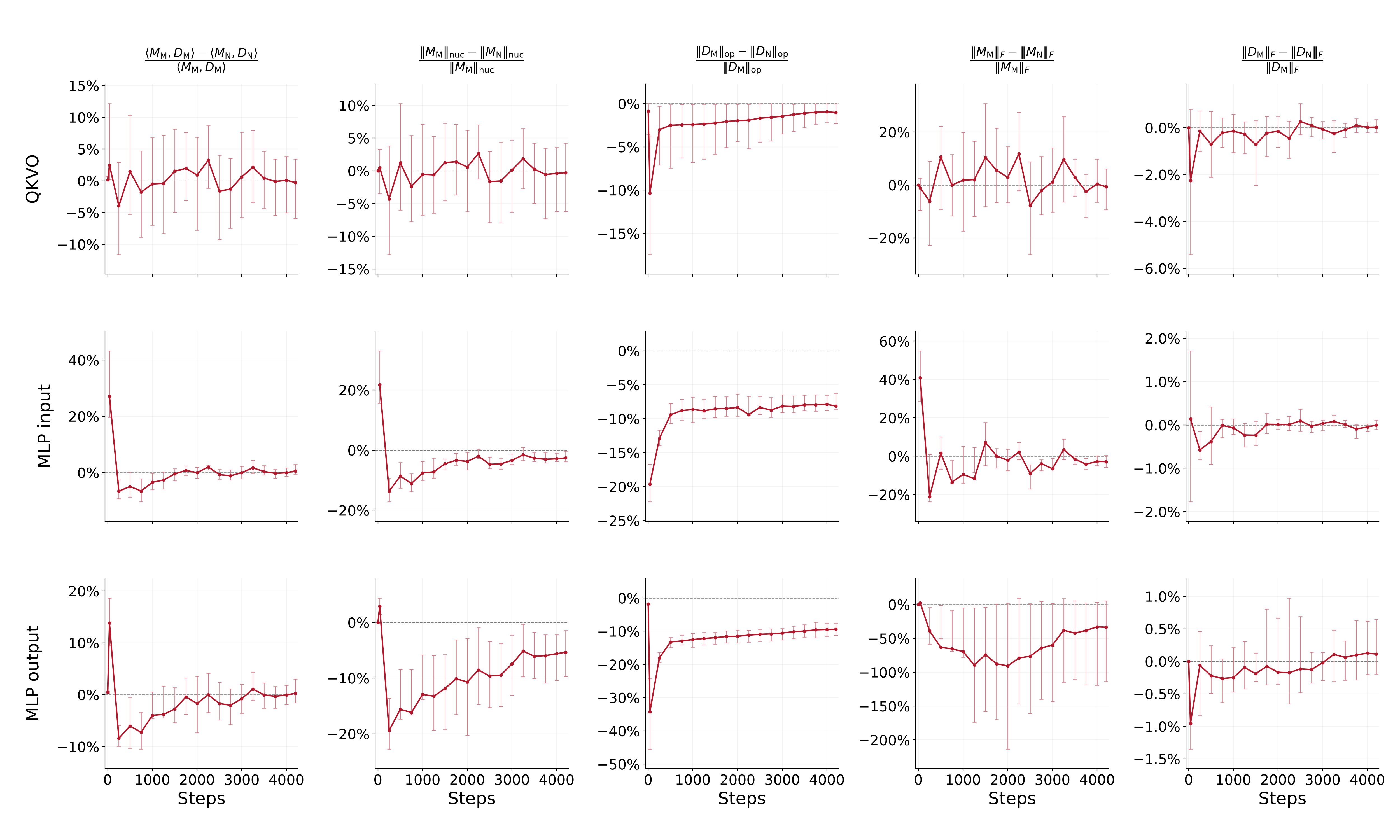}
\caption{Relative differences between Muon and NorMuon in the alignment and in the norms of the momentum and update direction, for different parameter groups. Error bars show the 10th--90th percentile range across the matrices in each group.}
\label{fig:alignment}
\end{figure}

\end{document}